\documentclass[11pt]{article}
\usepackage{authblk}

\newlength{\defbaselineskip}
\usepackage{hyperref}

\usepackage{amsmath}
\usepackage{amsthm} 
\usepackage{amssymb}
\usepackage{stmaryrd}

\DeclareMathOperator*{\argmin}{arg\,min}
\DeclareMathOperator{\diag}{diag}
\DeclareMathOperator{\trace}{tr}
\DeclareMathOperator{\rank}{rank}
\newcommand{\dataset}[1]{\texttt{#1}}
\newcommand{\defeq}{\stackrel{\text{\tiny def}}{=}}

\newcommand{\order}{\mathcal{O}} 
\newcommand{\x}{\mathbf{x}}
\renewcommand{\xi}{{\x}_{i}}

\newcommand{\z}{\mathbf{z}}
\newcommand{\X}{\mathbf{X}}
\newcommand{\QQ}{\mathbf{Q}}
\newcommand{\RR}{\mathbf{R}}
\newcommand{\Y}{\mathbf{Y}}
\newcommand{\K}{\mathbf{K}}
\newcommand{\G}{\mathbf{G}}

\newcommand{\eye}{\mathbf{I}}
\newcommand{\R}{\mathbb{R}}    
\newcommand{\LL}{\mathbf{L}}
\newcommand{\SSS}{\mathbf{S}}

\newcommand{\Z}{\mathbf{Z}}

\newcommand{\Lam}{\boldsymbol{\Lambda}}

\newcommand{\CC}{\mathbf{C}}
\newcommand{\WW}{\mathbf{W}}
\newcommand{\EE}{\mathbf{E}}
\newcommand{\PP}{\mathbf{P}}
\newcommand{\UU}{\mathbf{U}}
\newcommand{\VV}{\mathbf{V}}
\newcommand{\SIGMA}{\boldsymbol{\Sigma}}
\newcommand{\HH}{\mathbf{H}}

\usepackage{xcolor} 

\theoremstyle{plain}
\newtheorem{thm}{\protect\theoremname}
\theoremstyle{plain}
\newtheorem{lem}[thm]{\protect\lemmaname}
\providecommand{\lemmaname}{Lemma}
\providecommand{\theoremname}{Theorem}

\newtheorem{exa}[thm]{Example}
\newtheorem{remark}[thm]{Remark}

\usepackage{graphicx}
\usepackage[caption=false,font=footnotesize]{subfig}
\usepackage{pifont}

\usepackage{algorithmic}
\usepackage{algorithm}

\usepackage{booktabs}
\usepackage{siunitx}

\begin{document}

\title{
Improved Fixed-Rank Nystr\"om Approximation via QR Decomposition: Practical and Theoretical Aspects \footnote{This paper is accepted for publication in  \href{https://doi.org/10.1016/j.neucom.2019.06.070}{Neurocomputing}.}
}

\author{Farhad Pourkamali-Anaraki\\Department of Computer Science, University of Massachusetts Lowell, MA, USA \and Stephen Becker\\Department of Applied Mathematics, University of Colorado Boulder, CO, USA}

\date{\vspace{-5ex}}


\maketitle

\begin{abstract}
	The Nystr\"om method is a popular technique that uses a small number of landmark points to compute a fixed-rank approximation of large kernel matrices that arise in machine learning problems. In practice, to ensure high quality approximations, the number of landmark points is chosen to be greater than the target rank. However, for simplicity the standard Nystr\"om method uses a sub-optimal procedure for rank reduction. In this paper, we examine the drawbacks of the standard Nystr\"om method in terms of poor performance and lack of theoretical guarantees. To address these issues, we present an efficient modification for generating improved fixed-rank Nystr\"om approximations. Theoretical analysis and numerical experiments are provided to demonstrate the advantages of the modified method over the standard Nystr\"om method. Overall, the aim of this paper is to convince researchers to use the modified method, as it has nearly identical computational complexity, is easy to code, has greatly improved accuracy in many cases, and is optimal in a sense that we make precise. 
\end{abstract}

\section{Introduction}
\label{sec:intro}
Kernel methods are widely used in various machine learning problems. Well-known examples include support vector machines \cite{VapnikSVM,suykens1999least}, kernel clustering \cite{girolami2002mercer,chitta2011approximate,pourkamali2016randomized,LANGONE2017}, and kernel ridge regression \cite{saunders1998ridge,hsieh2014fast,alaoui2015fast,yang2017randomized}. The main idea
behind kernel-based learning is to map the input data points into a feature space, where all pairwise inner products of the mapped data points can be computed via a nonlinear kernel function that satisfies  Mercer's condition \cite{LearningWithKernels}. Thus, kernel methods allow one to use linear algorithms in the feature space which correspond
to nonlinear algorithms in the original space. For this reason, kernel machines have received much attention as an effective tool to tackle problems with complex and nonlinear structures.

Let $\x_1,\ldots,\x_n$ be a set of $n$ data points in $\R^p$. The inner products in feature space are calculated using a nonlinear kernel function $\kappa(\cdot,\cdot)$:
\begin{equation} \label{eq:kernel}
K_{ij}\defeq\kappa(\x_i,\x_j)=\langle\Phi(\x_i),\Phi(\x_j)\rangle,\;\;\forall i,j\in\{1,\ldots,n\}, 
\end{equation}
where $\Phi:\x\mapsto\Phi(\x)$ is the kernel-induced feature map. A popular choice is the Gaussian kernel function $\kappa(\x_i,\x_j)=\exp(-\|\x_i-\x_j\|_2^2/c)$, with the parameter $c>0$. In kernel machines, the pairwise inner products are stored in the symmetric positive semidefinite (SPSD) kernel matrix $\K\in\R^{n\times n}$. However, it takes $\order(n^2)$ memory to store the full kernel matrix and subsequent processing of $\K$ within the learning process is quite expensive or prohibitive for large data sets.

A popular approach to tackle these challenges is to use the best rank-$r$ approximation $\llbracket\K\rrbracket_r=\UU_r\Lam_r\UU_r^T$, obtained via the eigenvalue decomposition of $\K$, for $r\leq\rank(\K)$. Here, the columns of $\UU_r\in\R^{n\times r}$ span the top $r$-dimensional eigenspace of $\K$, and the diagonal matrix $\Lam_r\in\R^{r\times r}$ contains the top $r$ eigenvalues. Since the kernel matrix is SPSD,  we have:
\begin{equation}
\K\approx\llbracket\K\rrbracket_r=\UU_r\Lam_r\UU_r^T=\LL\LL^T,\label{eq:low-rank-kernel}
\end{equation}
where $\LL\defeq\UU_r\Lam_r^{1/2}\in\R^{n\times r}$.

When the target rank $r$ is small and chosen independently of $n$ (e.g., $r$ is chosen according to the degrees of freedom in the learning problem \cite{BachKernelReview}), the benefits of the rank-$r$ approximation in \eqref{eq:low-rank-kernel} are twofold. First, it takes $\order(nr)$ to store the matrix $\LL$ which is only linear in the number of samples $n$. Second, the rank-$r$ approximation leads to substantial computational savings within the learning process.  For example, approximating $\K$ with $\LL\LL^T$ means the matrix inversion $\left(\K+\lambda\eye_{n\times n}\right)^{-1}$ in kernel ridge regression can be calculated using the Sherman-Morrison-Woodbury formula 
in $\order(nr^2+r^3)$ time compared to $\order(n^3)$ if done na\"ively.
Other examples are
kernel K-means clustering, which is performed on the columns of the matrix $\LL^T\in\R^{r\times n}$, and so each step of the K-means algorithm runs in time proportional to $r$.

Although it has been shown that the fixed-rank approximation of kernel matrices is a promising approach to trade-off accuracy for
scalability  \cite{cortes2010impact,LinearizedSVM,golts2016linearized,wang2016towards}, the eigenvalue decomposition of $\K$ has at least quadratic time complexity and takes $\order(n^2)$ space. To address this issue, one line of prior work is centered around efficient techniques for approximating the best rank-$r$ approximation when we have ready access to $\K$; see \cite{Martinson_SVD,FarhadPreconditioned,tropp2017practical} for a survey. 
However, $\K$ is typically
unknown in kernel methods and the cost to form $\K$ using standard kernel functions is $\order(pn^2)$, which is extremely expensive for large high-dimensional data sets. For this reason, the Nystr\"om method \cite{Nystrom2001} has been a popular technique for computing fixed-rank approximations, which eliminates the need to access every entry of the full kernel matrix. The Nystr\"om method works by selecting a small set of vectors, referred to as landmark points, and computes the kernel similarities between the input data points and landmark points. 

To be formal, the standard Nystr\"om method generates a rank-$r$ approximation of $\K$ using $m$ landmark points $\z_1,\ldots,\z_m$ in $\R^p$. In practice, it is common to choose $m$ greater than $r$ for obtaining higher quality rank-$r$ approximations \cite{kumar2012sampling,li2015large}, since the accuracy of the Nystr\"om method depends on the number of selected landmark points and the selection procedure.
The landmark points can be sampled with respect to a uniform or nonuniform distribution from the set of $n$ input data points \cite{gittens2016revisiting,musco2017recursive}. Moreover, some recent techniques utilize out-of-sample landmark points for generating improved Nystr\"om approximations, e.g., centroids found from K-means clustering on the input data points \cite{zhang2008improved,zhang2010clusteredNys,KernelKmeansNystrom,FarhadAAAI}. For a fixed set of landmark points, let $\CC\in\R^{n\times m}$ and $\WW\in\R^{m\times m}$ be two matrices with the $(i,j)$-th entries $C_{ij}=\kappa(\x_i,\z_j)$ and $W_{ij}=\kappa(\z_i,\z_j)$. Then, the rank-$m$ Nystr\"om approximation has the form $\G=\CC\WW^\dagger\CC^T$, where $\WW^\dagger$ is the pseudo-inverse of $\WW$. For the fixed-rank case, the standard Nystr\"om method restricts the rank of the $m\times m$ inner matrix $\WW$ and computes its best rank-$r$ approximation $\llbracket\WW\rrbracket_r$ to obtain $\G_{(r)}^{nys}=\CC\llbracket\WW\rrbracket_r^\dagger\CC^T$, which has rank no great than $r$.
The manner of the  rank-$m$ to rank-$r$ reduction may appear {\em ad hoc}, and {\em improving this reduction is the topic of the paper.}

Although the rank reduction process in the standard Nystr\"om method is simple, 
it disregards the structure of $\CC$.
This method generates the rank-$r$ approximation $\G_{(r)}^{nys}$ solely based on filtering $\WW$ because of its smaller size compared to the matrix $\CC$ of size $n\times m$. As a result, {\em the selection of more distinct landmark points in the standard Nystr\"om method does not guarantee improved rank-$r$ approximations of kernel matrices}. For example, our experimental results in Section \ref{sec:exper} reveal that the increase in the number of landmark points may even produce less accurate rank-$r$ approximations due to the poor rank reduction process, cf.~Remark \ref{rmk:1} and Remark \ref{rmk:2}.

This paper considers the fundamental problem of rank reduction in the Nystr\"om method. In particular, we present an efficient technique for computing a rank-$r$ approximation in the form of $\G_{(r)}^{opt}=\llbracket\CC\WW^\dagger\CC^T\rrbracket_r$, which runs in time comparable with the standard Nystr\"om method. The modified method utilizes the thin QR decomposition of the matrix $\CC$ for computing a more accurate rank-$r$ approximation of $\CC\WW^\dagger\CC^T$ compared to $\G_{(r)}^{nys}$. Moreover, unlike the standard Nystr\"om method, our results show that both theoretically and empirically, modified Nystr\"om produces more accurate rank-$r$ approximations as the number of distinct landmark points increases. 
\subsection{Contributions}
In this work, we make the following contributions:
\begin{enumerate}
	\item In Algorithm \ref{alg:NysQR}, we present an efficient method for generating improved rank-$r$ Nystr\"om approximations. The modified method computes the best rank-$r$ approximation of $\CC\WW^\dagger\CC^T$, i.e., $\G_{(r)}^{opt}=\llbracket\CC\WW^\dagger\CC^T\rrbracket_r$, in linear time with respect to the sample size $n$. In Theorem \ref{thm:nys-qr-sta}, it is shown that $\G_{(r)}^{opt}$ always produces a more accurate rank-$r$ approximation of $\K$ compared to $\G_{(r)}^{nys}$ with respect to the trace norm, when $m$ is greater than  $r$ and landmark points are selected from the input data set. Remark
	\ref{thm:remark-frob} shows this is not necessarily true in the Frobenius norm, although it is rarely seen in practice.
	\item Theorem \ref{thm:nys-qr-std-more} proves that the accuracy of the modified rank-$r$ Nystr\"om approximation always improves (with respect to the trace norm) as
	more distinct landmark points are selected from the input data set. 
	\item We provide counter-examples in Remark  \ref{rmk:1} and Remark \ref{rmk:2} showing that an equivalent of Theorem \ref{thm:nys-qr-std-more} cannot hold for the standard Nystr\"om method. Example \ref{example1} shows a situation where the modified Nystr\"om method is arbitrarily better than the standard method, with respect to the trace and Frobenius norms. Remark \ref{rmk:same} gives insight into when we expect the standard and modified methods to differ. 
	\item Theorem \ref{thm:out-of-sample} shows that, under certain conditions, our theoretical results are also applicable to more recent selection techniques based on out-of-sample extensions of the input data, such as centroids found from K-means clustering.
	\item Finally, we provide experimental results to demonstrate the superior performance and advantages of modified Nystr\"om.
\end{enumerate}

To our knowledge, the modified Nystr\"om method was not discussed in the literature until our preliminary preprint \cite{pourkamali2016rand}, though its derivation is straightforward and so we suspect it may have been previously derived in unpublished work; our main contribution is the mathematical analysis. Due to the importance of rank reduction in the Nystr\"om method, there are two recent works \cite{wang2017scalable,tropp2017fixed} that independently study the approximation error of $\llbracket\CC\WW^\dagger\CC^T\rrbracket_r$, when landmark points are selected from the input data set. However, there are two principal differences between this work and the aforementioned references. First, the main focus of this paper is to directly compare the standard and modified Nystr\"om methods, and provide both theoretical and experimental evidences on the effectiveness of modified Nystr\"om, while \cite{wang2017scalable,tropp2017fixed} do not provide results comparing the two methods.
Second, we present theoretical results for the important class of out-of-sample landmark points, which often lead to accurate Nystr\"om approximations.

\subsection{Paper Organization}
In  Section \ref{sec:notation}, we present the notation and give a brief review of some matrix decomposition and low-rank approximation techniques. Section \ref{sec:standard-nys} reviews the standard Nystr\"om method for computing  rank-$r$ approximations and we explain the process of obtaining approximate eigenvalues and eigenvectors. In  Section \ref{sec:improved-nys}, we present an efficient modified method for computing improved rank-$r$ approximations of kernel matrices. The main theoretical results are given in Section \ref{sec:theory} and Section \ref{sec:thm-outofsample}, and we present experimental results comparing the modified and standard Nystr\"om methods in Section \ref{sec:exper}. Section \ref{sec:conclusion} provides a brief conclusion.

\section{Notation and Preliminaries}\label{sec:notation}
We denote column vectors with lower-case bold letters and matrices with upper-case bold letters. $\eye_{n\times n}$ is the identity matrix of size $n\times n$; $\mathbf{0}_{n\times m}$ is the $n\times m$ matrix of zeros. For a vector $\x\in\R^p$, let $\|\x\|_2$ denote the  Euclidean norm, and $\diag(\x)$ represents a diagonal matrix with the elements of $\x$ on the main diagonal. The $(i,j)$-th entry of $\mathbf{A}$ is denoted by $A_{ij}$, $\mathbf{A}^T$ is the transpose of $\mathbf{A}$, and $\trace(\cdot)$ is the trace operator.
We assume scalars, vectors and matrices are real-valued, though many of the results extend to complex numbers.

Each $n\times m$ matrix $\mathbf{A}$ with $\rho=\rank(\mathbf{A})\leq \min\{n,m\}$ admits a factorization in the form of $\mathbf{A}=\UU\SIGMA\VV^T$, where $\UU\in\R^{n\times \rho}$ and $\VV\in\R^{m\times \rho}$ are orthonormal matrices known as the left singular vectors and right singular vectors, respectively. The diagonal matrix $\SIGMA=\diag([\sigma_1(\mathbf{A}),\ldots,\sigma_\rho(\mathbf{A})])$ contains the singular values of $\mathbf{A}$ in descending order, i.e., $\sigma_1(\mathbf{A})\geq\ldots\geq\sigma_\rho(\mathbf{A})>0$. 
This factorization is known as the \emph{thin} singular value decomposition (SVD).

Throughout the paper, we use several standard matrix norms. The Frobenius norm of $\mathbf{A}$ is defined as
$
\|\mathbf{A}\|_F^2\defeq\sum_{i=1}^{\rho}\sigma_i(\mathbf{A})^2=\trace(\mathbf{A}^T\mathbf{A})
$ 
and $\|\mathbf{A}\|_*\defeq\sum_{i=1}^{\rho}\sigma_i(\mathbf{A})=\trace(\sqrt{\mathbf{A}^T\mathbf{A}})$ denotes the trace norm (or nuclear norm) of $\mathbf{A}$. The spectral  norm of $\mathbf{A}$ is the largest singular value of $\mathbf{A}$, i.e., $\|\mathbf{A}\|_2\defeq\sigma_1(\mathbf{A})$.  
It is straightforward to show that $\|\mathbf{A}\|_2\leq \|\mathbf{A}\|_F\leq \|\mathbf{A}\|_*$. 

When $\K$ is a \emph{kernel matrix}, meaning it is generated via \eqref{eq:kernel}, we assume the kernel function $\kappa$ satisfies Mercer's condition and therefore $\K$ is symmetric positive semidefinite (SPSD)~\cite{aronszajn1950theory,LearningWithKernels}. 
Let $\K\in\R^{n\times n}$ be any SPSD matrix with
$\rho=\text{rank}(\K)\leq n$.
Similar to the SVD, the matrix $\K$ can be factorized as $\K=\UU\Lam\UU^T$, where $\UU\in\R^{n\times \rho}$ contains the orthonormal eigenvectors, i.e., $\UU^T\UU=\eye_{\rho\times \rho}$, and  $\Lam=\diag\left([\lambda_1(\K),\ldots,\lambda_\rho(\K)]\right)$ is a diagonal matrix which contains the nonzero eigenvalues of $\K$ in descending order. This factorization is known as the \emph{thin} eigenvalue decomposition (EVD). The matrices $\UU$ and $\Lam$ can be partitioned for a target rank $r$ ($r\leq\rho$) in the form of $\K=\UU_r\Lam_r\UU_r^T+\UU_{\rho-r}\Lam_{\rho-r}\UU_{\rho-r}^T$,
where $\Lam_r\in\R^{r\times r}$ contains the $r$ leading eigenvalues and the columns of $\UU_r\in\R^{n\times r}$ span the top $r$-dimensional eigenspace, and $\Lam_{\rho-r}\in\R^{(\rho-r)\times (\rho-r)}$ and $\UU_{\rho-r}\in\R^{n\times (\rho - r)}$ contain the remaining $(\rho-r)$ eigenvalues and eigenvectors. 
It is well-known that $\llbracket\K\rrbracket_r\defeq\UU_r\Lam_r\UU_r^T$ is the ``best rank-$r$ approximation'' to $\K$ in the sense that $\llbracket\K\rrbracket_r$ minimizes $\|\K-\K'\|_F$ and $\|\K-\K'\|_*$ over all matrices $\K'\in\R^{n\times n}$ of rank at most $r$. If $\lambda_r(\K) = \lambda_{r+1}(\K)$, then $\llbracket\K\rrbracket_r$ is not unique. The Moore-Penrose pseudo-inverse of $\K$ can be obtained from the EVD as $\K^\dagger=\UU\Lam^{-1}\UU^T$. When $\K$ is full rank, we have $\K^\dagger=\K^{-1}$.

Another matrix factorization technique that we use is the QR decomposition. An $n\times m$ matrix $\mathbf{A}$, with $n\geq m$, can be decomposed as $\mathbf{A}=\QQ\RR$,
where $\QQ\in\R^{n\times m}$ has $m$ orthonormal columns, i.e., $\QQ^T\QQ=\eye_{m\times m}$, and $\RR\in\R^{m\times m}$ is an upper triangular matrix.
Sometimes this is called the \emph{thin} QR decomposition, to distinguish it from a \emph{full} QR decomposition which finds $\QQ\in\R^{n\times n}$ and zero-pads $\RR$ accordingly.

Finally, we state a standard result on the rank-$r$ approximation of a matrix  expressed as a product of two matrices. The proof of this result can be found in \cite{boutsidis2014near}.
\begin{lem}\label{lemma:best-rank-ONB} Consider the matrix  $\K\in\R^{n\times n}$ and let $\QQ\in\R^{n\times m}$ be a matrix that has $m<n$ orthonormal columns. For any positive integer $r\leq m$, we have:
	\begin{equation}
	\llbracket\QQ^T\K\rrbracket_r = \argmin_{\mathbf{T}:\; \rank(\mathbf{T})\leq r}\| \K - \QQ\mathbf{T}\|_F^2.
	\end{equation}
\end{lem}

\section{The Standard Nystr\"om Method}\label{sec:standard-nys}
The Nystr\"om method generates a fixed-rank approximation of the SPSD kernel matrix $\K\in\R^{n\times n}$  by selecting a small set of vectors referred to as ``landmark points''. The simplest selection technique is uniform sampling without replacement \cite{Nystrom2001,kumar2012sampling}, where each data point is sampled with the same probability, i.e., $p_i=\frac{1}{n}$, for $i=1,\ldots,n$. The advantage of this technique is the low computational complexity associated with sampling landmark points. However, uniform sampling does not take into account the nonuniform structure of many data sets and the resulting kernel matrices. Therefore, sampling mechanisms with respect to nonuniform  distributions have been proposed to address this problem. This line of work requires the computation of statistical leverage scores of $\K$, which is more expensive than uniform sampling \cite{Nystrom_Kernel_Approx,mahoney2009cur,FastApproxCohLev}. In addition, leverage score sampling often requires computing the entire kernel matrix $\K$, which negates one of the principal benefits of the Nystr\"om method.  A comprehensive review and comparison of uniform and nonuniform landmark selection techniques can be found in \cite{kumar2012sampling,sun2015review}.  

More recently, generating landmark points using out-of-sample extensions of input data has been shown to be effective for high quality Nystr\"om approximations. This line of research originates from the work of Zhang et al.~\cite{zhang2008improved,zhang2010clusteredNys}, and it is based on the observation that the Nystr\"om approximation error depends on the quantization error of encoding the data set with the landmark points. Hence, the landmark points are selected to be the centroids found from K-means clustering. In machine learning and pattern recognition, K-means clustering is a well-established technique to partition a data set into clusters 
by trying to minimize the total sum of the squared Euclidean distances of each point to the closest cluster center \cite{Bishop}. 

In general, assume that a set of $m\ll n$ landmark points in $\R^p$, denoted by  $\Z=[\z_1,\ldots,\z_m]\in\R^{p\times m}$, are given. Let us consider two matrices $\CC\in\R^{n\times m}$ and $\WW\in\R^{m\times m}$, where $C_{ij}=\kappa(\x_i,\z_j)$ and $W_{ij}=\kappa(\z_i,\z_j)$.
The Nystr\"om method uses both $\CC$ and $\WW$ to construct an
approximation of the kernel matrix $\K$ in the form of $\K\approx\G=\CC\WW^\dagger\CC^T$, which has rank at most $m$.
For the fixed-rank case, the Nystr\"om method generates a rank-$r$ approximation of the kernel matrix, $r\leq m$, by computing the best rank-$r$ approximation of the inner matrix $\WW$ \cite{li2015large,gittens2016revisiting,lu2016large}, which results in 
$\G_{(r)}^{nys}=\CC\llbracket\WW\rrbracket_r^\dagger\CC^T$,
where $\llbracket\WW\rrbracket_r^\dagger$ represents the pseudo-inverse of $\llbracket\WW\rrbracket_r$. Thus, the EVD of the matrix $\WW=\VV\SIGMA\VV^T$ should be computed to find the top $r$ eigenvalues $\SIGMA_r\in\R^{r\times r}$  and corresponding eigenvectors $\VV_r\in\R^{m\times r}$:
\begin{equation}
\G_{(r)}^{nys}=\LL^{nys}\left(\LL^{nys}\right)^T,\; \LL^{nys}=\CC\VV_r\Big(\SIGMA_r^{-1}\Big)^{1/2}.\label{eq:Nys-low-rank-ll}
\end{equation}
The time complexity of the Nystr\"om
method to form $\LL^{nys}$ is $\order(pnm+m^2r+nmr)$, where it takes $\order(pnm)$ to construct $\CC$ and $\WW$. 
It takes $\order(m^2r)$ time to perform the partial EVD of $\WW$ and $\order(nmr)$ represents the cost of the matrix multiplication $\CC\VV_r$. Thus, for $r\leq m\ll n$, the computation cost to form the rank-$r$ approximation of the kernel matrix is only linear in the data set size $n$. 
The eigenvalues and eigenvectors of $\K$ can be estimated by using the rank-$r$ approximation in \eqref{eq:Nys-low-rank-ll},
and in fact this approach provides the exact eigenvalue decomposition of $\G_{(r)}^{nys}$. 
The first step is to find the EVD of the $r\times r$ matrix: $(\LL^{nys})^T\LL^{nys}=\widetilde{\VV}\widetilde{\SIGMA}\widetilde{\VV}^T$, where $\widetilde{\VV},\widetilde{\SIGMA}\in\R^{r\times r}$. Then, the estimates of $r$ leading eigenvalues and eigenvectors of $\K$ are obtained as: $\widehat{\UU}_{r}^{nys}=\LL^{nys}\widetilde{\VV}\big(\widetilde{\SIGMA}^{-1}\big)^{1/2}$ and $\widehat{\boldsymbol{\Lambda}}_r^{nys}=\widetilde{\SIGMA}$. The procedure to estimate the $r$ leading eigenvalues/eigenvectors is summarized in Algorithm \ref{alg:StandardNys}. 

\begin{algorithm}[t]
	\caption{Standard Nystr\"om}
	\label{alg:StandardNys}
	\textbf{Input:} data set $\X$, $m$ landmark points $\Z$, kernel function $\kappa$, target rank $r$
	
	\textbf{Output:} estimates of $r$ leading eigenvectors and eigenvalues of the kernel matrix $\K\in\R^{n\times n}$: $\widehat{\UU}_r^{nys}\in\R^{n\times r}$, $\widehat{\Lam}_r^{nys}\in\R^{r\times r}$
	\begin{algorithmic}[1]
		\STATE  Form $\CC$ and $\WW$: $C_{ij}=\kappa(\x_i,\z_j)$, $W_{ij}=\kappa(\z_i,\z_j)$
		\STATE Compute EVD: $\WW=\VV\SIGMA\VV^T$
		\STATE Form the matrix: $\LL^{nys}=\CC\VV_r\big(\SIGMA_r^{-1}\big)^{1/2}$ 
		\STATE Compute EVD: $(\LL^{nys})^T\LL^{nys}=\widetilde{\VV}\widetilde{\SIGMA}\widetilde{\VV}^T$
		\STATE $\widehat{\UU}_{r}^{nys}=\LL^{nys}\widetilde{\VV}\Big(\widetilde{\SIGMA}^{-1}\Big)^{1/2}$ and $\widehat{\boldsymbol{\Lambda}}_r^{nys}=\widetilde{\SIGMA}$
	\end{algorithmic}
\end{algorithm} 

\section{Improved Nystr\"om Approximation via QR Decomposition}\label{sec:improved-nys}
In the previous section, we explained the Nystr\"om method for computing rank-$r$ approximations of SPSD kernel matrices based on selecting a small set of landmark points. 
Although the final goal is to find an approximation that has rank no greater than $r$,  it is often preferred to select $m>r$ landmark points and then restrict the resultant approximation to have rank at most $r$. The main intuition is that selecting $m>r$ landmark points and then restricting the approximation to a lower rank-$r$ space has a regularization effect which can lead to more accurate approximations \cite{kumar2012sampling,gittens2016revisiting}. For example, when landmark points are chosen to be centroids from K-means clustering, more landmark points lead to smaller quantization error of the data set, and thus higher quality Nystr\"om approximations. 

In the standard Nystr\"om method presented in Algorithm \ref{alg:StandardNys}, the rank of the matrix $\G$ is restricted by computing the best rank-$r$ approximation of the inner matrix $\WW$: $\G_{(r)}^{nys}=\CC\llbracket\WW\rrbracket_r^\dagger\CC^T$. Since the inner matrix in the representation of $\G_{(r)}^{nys}$ has rank no greater than $r$, it follows that $\G_{(r)}^{nys}$ has rank at most $r$. The main benefit of this technique is the low computational cost of performing an exact eigenvalue decomposition on a relatively small matrix of size $m\times m$. However, the standard Nystr\"om method ignores the structure of the matrix $\CC$ in the rank reduction process. In fact, since the rank-$r$ approximation $\G_{(r)}^{nys}$ does not utilize the full knowledge of $\CC$, the selection of more landmark points does not guarantee an improved low-rank approximation in the standard Nystr\"om method, cf.~Remark \ref{rmk:1} and Remark \ref{rmk:2}.

To solve this problem, we present an efficient method to compute the best rank-$r$ approximation of $\G=\CC\WW^\dagger\CC^T$, for given matrices $\CC\in\R^{n\times m}$ and $\WW\in\R^{m\times m}$. In contrast with the standard Nystr\"om method, the modified approach takes advantage of both matrices $\CC$ and $\WW$. To begin, let us consider the best
rank-$r$ approximation of $\G$ in any unitarily invariant norm $\|\cdot\|$, such as the Frobenius norm or trace norm: 
\begin{align}
\G_{(r)}^{opt}&\defeq \argmin_{\G':\;\text{rank}(\G')\leq r} \|\CC\WW^\dagger\CC^T-\G'\| 
\overset{(i)}{=}   \argmin_{\G':\;\text{rank}(\G')\leq r} \|\QQ\underbrace{\RR\WW^\dagger\RR^T}_{m\times m}\QQ^T-\G'\|\nonumber\\
&\overset{(ii)}{=}  \argmin_{\G':\;\text{rank}(\G')\leq r} \|\left(\QQ\VV'\right)\SIGMA'\left(\QQ\VV'\right)^T-\G'\| =  \left(\QQ \VV'_r\right) \SIGMA'_r\left(\QQ \VV'_r\right)^T,\label{eq:optimal-nys}
\end{align}
where (i) follows from the QR decomposition of $\CC\in\R^{n\times m}$; $\CC=\QQ\RR$, where $\QQ\in\R^{n\times m}$ and $\RR\in\R^{m\times m}$. To get (ii), the EVD of the $m\times m$ matrix $\RR\WW^\dagger\RR^T$ is computed: $\RR\WW^\dagger\RR^T=\VV'\SIGMA'\VV'^T$, where the diagonal matrix $\SIGMA'\in\R^{m\times m}$ contains $m$ eigenvalues in descending order on the main diagonal and the columns of $\VV'\in\R^{m\times m}$ are the corresponding eigenvectors. Moreover, we note that the columns of  $\QQ\VV'\in\R^{n\times m}$ are orthonormal because both $\QQ$ and $\VV'$ have orthonormal columns. 
Thus, the decomposition $(\QQ\VV')\SIGMA'(\QQ\VV')^T$ contains the $m$ eigenvalues and orthonormal eigenvectors of the Nystr\"om approximation $\CC\WW^\dagger\CC^T$. Hence, the best rank-$r$ approximation of $\G=\CC\WW^\dagger\CC^T$ is then computed using the $r$ leading eigenvalues $\SIGMA'_r\in\R^{r\times r}$ and corresponding eigenvectors $\QQ\VV'_r\in\R^{n\times r}$, as given \eqref{eq:optimal-nys}. Thus, the estimates of the top $r$ eigenvalues and eigenvectors of the kernel matrix $\K$ from the Nystr\"om approximation $\CC\WW^\dagger\CC^T$ are obtained as: $\widehat{\UU}_r^{opt}=\QQ\VV'_r$ and $\widehat{\Lam}_r^{opt}=\SIGMA'_r$.  

The modified method for estimating the $r$ leading eigenvalues/eigenvectors of the kernel matrix is presented in Algorithm \ref{alg:NysQR}. The time complexity is $\order(pnm+nm^2+m^3+nmr)$, where $\order(pnm)$ represents the cost to form $\CC$ and $\WW$.
The complexity of the QR decomposition is $\order(nm^2)$ and it takes $\order(m^3)$ time to compute the EVD of $\RR\WW^\dagger\RR^T$. Finally, the cost to compute the matrix multiplication $\QQ\VV'_r$ is $\order(nmr)$.
\begin{algorithm}[t]
	\caption{Nystr\"om via QR Decomposition (``modified'' Nystr\"om)}
	\label{alg:NysQR}
	\textbf{Input:} data set $\X$, $m$ landmark points $\Z$, kernel function $\kappa$, target rank $r$
	
	\textbf{Output:} estimates of $r$ leading eigenvectors and eigenvalues of the kernel matrix $\K\in\R^{n\times n}$: $\widehat{\UU}_r^{opt}\in\R^{n\times r}$, $\widehat{\Lam}_r^{opt}\in\R^{r\times r}$
	\begin{algorithmic}[1]
		\STATE 	  Form $\CC$ and $\WW$: $C_{ij}=\kappa(\x_i,\z_j)$, $W_{ij}=\kappa(\z_i,\z_j)$
		\STATE Perform the thin QR decomposition: $\CC=\QQ\RR$
		\STATE Compute EVD: $\RR\WW^\dagger\RR^T=\VV'\SIGMA'\VV'^T$
		\STATE $\widehat{\UU}_r^{opt}=\QQ\VV'_r$ and $\widehat{\Lam}_r^{opt}=\SIGMA'_r$
	\end{algorithmic}
\end{algorithm}

We can compare the computational complexity of Nystr\"om via QR decomposition with that of the standard Nystr\"om method. Since our focus in this paper is on large-scale data sets with $n$ large, we only consider terms involving $n$ which lead to dominant computation costs. Based on our previous discussion, it takes $\mathcal{C}_{nys}=\order(pnm+nmr+nr^2)$ time to compute the eigenvalue decomposition using the standard Nystr\"om method, while the cost of the modified technique is $\mathcal{C}_{opt}=\order(pnm+nmr+nm^2)$. Thus, for data of even moderate dimension with $p\gtrsim m$, the dominant term in both $\mathcal{C}_{nys}$ and $\mathcal{C}_{opt}$ is $\order(pnm)$.
Hence, there is no significant increase in cost, as is the case in our runtime example shown in Figure \ref{fig:runtime}.

\begin{exa} \label{example1}
	In the rest of this section, we present a simple example to gain some intuition on the superior performance of the modified technique.
\end{exa}
Let us consider a small kernel matrix of size $3\times 3$: 
\begin{equation}
\K=\begin{bmatrix}
1 & 0 & 10\\
0 & 1.01 & 0\\
10 & 0 & 100\end{bmatrix}.\label{eq:ex-kernel}
\end{equation}
One can find, for example, a data matrix $\X$ that generates this kernel matrix as $\K=\X^T\X$.
Here, the goal is to compute the rank $r=1$ approximation of $\K$. Sample  $m=2$ columns of $\K$, and suppose we choose the first and second columns:
\begin{equation}
\CC=\begin{bmatrix}
1 & 0\\
0 & 1.01\\
10 & 0
\end{bmatrix},\;\;\WW=\begin{bmatrix}
1 & 0\\
0 & 1.01\end{bmatrix}.
\end{equation}
In the standard Nystr\"om method, the best rank-$1$ approximation of the matrix $\WW$ is first computed. Then, the rank-$1$ approximation of $\K$ using standard Nystr\"om is:
\begin{equation}
\G_{(1)}^{nys}
= \CC\llbracket\WW\rrbracket_1^\dagger\CC^T=
\begin{bmatrix}
0 & 0 & 0\\
0 & 1.01 & 0\\
0 & 0 & 0
\end{bmatrix}.
\end{equation}
The normalized  approximation error in terms of the Frobenius norm and trace norm is large:  $\|\K-\G_{(1)}^{nys}\|_F/\|\K\|_F=0.99$ and $\|\K-\G_{(1)}^{nys}\|_*/\|\K\|_*=0.99$.
On the other hand,  the modified method first computes the QR decomposition of $\CC=\QQ\RR$:
\begin{equation}
\QQ=\begin{bmatrix}
\frac{1}{\sqrt{101}} & 0\\
0 & 1\\
\frac{10}{\sqrt{101}} & 0
\end{bmatrix},\;\;\RR=\begin{bmatrix}
\sqrt{101} & 0\\
0 & 1.01
\end{bmatrix}.
\end{equation}
Then, the product of three matrices $\RR\WW^\dagger\RR^T$ is computed to find its EVD:
\begin{equation}
\RR\WW^\dagger\RR^T=  \begin{bmatrix}
\sqrt{101} & 0\\
0 & 1.01
\end{bmatrix}
\begin{bmatrix}
1 & 0\\
0 & \frac{1}{1.01}
\end{bmatrix}
\begin{bmatrix}
\sqrt{101} & 0\\
0 & 1.01
\end{bmatrix}
=  \underbrace{\begin{bmatrix}
	1 & 0\\
	0 & 1
	\end{bmatrix}}_{\VV'}\underbrace{\begin{bmatrix}
	101 & 0\\
	0 & 1.01
	\end{bmatrix}}_{\SIGMA'}\underbrace{\begin{bmatrix}
	1 & 0\\
	0 & 1
	\end{bmatrix}}_{\VV'^T}.
\end{equation}
Finally, the rank-$1$ approximation of the kernel matrix in the modified method is:
\begin{equation}
\G_{(1)}^{opt}=\begin{bmatrix}
\frac{1}{\sqrt{101}} & 0\\
0 & 1\\
\frac{10}{\sqrt{101}} & 0
\end{bmatrix}\begin{bmatrix}
101 & 0\\
0 & 0
\end{bmatrix}
\begin{bmatrix}
\frac{1}{\sqrt{101}} & 0 & \frac{10}{\sqrt{101}}\\
0 & 1 & 0
\end{bmatrix}
=\begin{bmatrix}
1 & 0 & 10\\
0 & 0 & 0\\
10 & 0 & 100
\end{bmatrix},
\end{equation}
where $\|\K-\G_{(1)}^{opt}\|_F/\|\K\|_F=0.01$ and $\|\K-\G_{(1)}^{opt}\|_*/\|\K\|_*=0.01$. 
In fact, one can show that our approximation is the same as the best rank-$1$ approximation of $\K$, i.e., $\G_{(1)}^{opt}=\llbracket\K\rrbracket_1$. 
Furthermore, by taking $K_{22} \searrow 1$ in \eqref{eq:ex-kernel}, we can make the improvement of the modified method over the standard method arbitrarily large.

\section{Main Theoretical Results}\label{sec:theory}
\newcommand{\D}{\mathbf{D}}
\newcommand{\M}{\mathbf{F}}
\newcommand{\NN}{\mathbf{N}}

In order to compare the accuracy of modified Nystr\"om with standard Nystr\"om, we first provide an alternative formulation of these two methods. 
We assume 
the landmark points $\Z=[\z_1,\ldots,\z_m]$ are \emph{in-sample}, meaning they are selected in any fashion (deterministic or random) from among the set of input data points,
so that 
the matrix $\CC\in\R^{n\times m}$ contains $m$ columns of the kernel matrix $\K$. This column selection process can be viewed as forming a \emph{sampling} matrix $\PP\in\R^{n\times m}$ that has exactly one nonzero entry in each column, where its location corresponds to the index of the selected landmark point. Then, the matrix product $\CC=\K\PP\in\R^{n\times m}$ contains $m$ columns sampled from the kernel matrix $\K$ and $\WW=\PP^T\K\PP\in\R^{m\times m}$ is the intersection of the $m$ columns with the corresponding $m$ rows of $\K$. 

Let us define $\D\defeq\K^{1/2}\PP\in\R^{n\times m}$, which means that $\CC=\K\PP=\K^{1/2}\D$ and $\WW=\PP^T\K\PP=\D^T\D$.
Moreover, we consider the SVD of $\D=\M\SSS\NN^T$, where the columns of  $\M\in\R^{n\times m}$ are the left singular vectors of $\D$, and we have $\SSS,\NN\in\R^{m\times m}$. Thus, we get the EVD of the matrix $\WW=\D^T\D=\NN\SSS^2\NN^T$.
For simplicity of presentation, we assume $\D$ and $\WW$ have full rank, though the results still hold as long as they have rank greater than or equal to $r$.

The rank-$r$ approximation in standard Nystr\"om is  $\G_{(r)}^{nys}=\LL^{nys}(\LL^{nys})^T$:
\begin{equation}
\LL^{nys}=\CC \left(\llbracket\WW\rrbracket_r^\dagger\right)^{1/2}=  \left(\K^{1/2}\M\SSS\NN^T\right)\left(\NN_r\SSS_r^{-2}\NN_r^T\right)^{1/2}=\K^{1/2} \M_r\NN_r^T,
\end{equation}
where we have used $(\NN_r\SSS_r^{-2}\NN_r^T)^{1/2}=\NN_r\SSS_r^{-1}\NN_r^T$, and the following two properties:
\begin{equation}
\NN^T\NN_r=\begin{bmatrix}
\eye_{r\times r}\\
\mathbf{0}_{(m-r)\times r}
\end{bmatrix},\;\;\;\M\begin{bmatrix}
\NN_r^T\\
\mathbf{0}_{(m-r)\times m}
\end{bmatrix}=\M_r\NN_r^T.
\end{equation}
Since the columns of $\NN_r$ are orthonormal, i.e., $\NN_r^T\NN_r=\eye_{r\times r}$, the rank-$r$ approximation of the kernel matrix $\K$ in the standard Nystr\"om method is given by:
\begin{equation}
\G_{(r)}^{nys}=\LL^{nys}\left(\LL^{nys}\right)^T=\K^{1/2}\M_r\M_r^T\K^{1/2}.\label{eq:G-nys-alt}
\end{equation}

Next, we present an alternative formulation of the rank-$r$ approximation $\G_{(r)}^{opt}$ in terms of the left singular vectors of $\D$. The modified Nystr\"om method finds the best rank-$r$ approximation of $\CC\WW^\dagger\CC^T$, and observe that:
\begin{equation}
\CC\big(\WW^\dagger\big)^{1/2}=\left(\K^{1/2}\M\SSS\NN^T\right)\left(\NN\SSS^{-1}\NN^T\right)=\K^{1/2}\M\NN^T.
\end{equation}
Thus, we get $\CC\WW^\dagger\CC^T=\K^{1/2}\M\M^T\K^{1/2}$, and the best rank-$r$ approximation is:
\begin{equation}
\G_{(r)}^{opt}=\llbracket\CC\WW^\dagger\CC^T\rrbracket_r=\llbracket\K^{1/2}\M\rrbracket_r\M^T\M\llbracket\M^T\K^{1/2}\rrbracket_r,\label{eq:nys-opt-v1}
\end{equation}
where we used $\M^T\M=\eye_{m\times m}$. Based on \cite[Lemma 6]{wang2017scalable}, let $\HH\in\R^{n\times r}$ be the orthonormal bases of the rank-$r$ matrix $\M\llbracket\M^T\K^{1/2}\rrbracket_r\in\R^{n\times n}$. Then, we have $\M\llbracket\M^T\K^{1/2}\rrbracket_r=\HH\HH^T\K^{1/2}$, which allows us to simplify \eqref{eq:nys-opt-v1}:
\begin{equation}
\G_{(r)}^{opt}=\K^{1/2}\HH\HH^T\K^{1/2}.\label{eq:alt-G-our}
\end{equation}
In the following, we present a theorem which shows that the modified Nystr\"om method generates improved rank-$r$ approximation of $\K$ compared to standard Nystr\"om. 
\begin{thm}[Modified Nystr\"om is more accurate than standard Nystr\"om]\label{thm:nys-qr-sta} Let $\K\in\R^{n\times n}$ be an SPSD kernel matrix, and $r$ be the target rank. 
	Let $\PP$ be any $n \times m$ matrix,
	with $m\geq r$, such that $\CC=\K\PP\in\R^{n\times m}$ and $\WW=\PP^T\K\PP\in\R^{m\times m}$. Then, we have:
	\begin{equation}
	\|\K-\G_{(r)}^{opt}\|_*\leq \|\K-\G_{(r)}^{nys}\|_*,
	\end{equation}
	where the Nystr\"om method via QR decomposition generates $\G_{(r)}^{opt}=\llbracket\CC\WW^\dagger\CC^T\rrbracket_r$, and the standard Nystr\"om method produces $\G_{(r)}^{nys}=\CC\llbracket\WW\rrbracket_r^\dagger\CC^T$.
\end{thm} 
\begin{proof}
	We start with the alternative formulation of $\G_{(r)}^{opt}=\K^{1/2}\HH\HH^T\K^{1/2}$, where the columns of $\HH\in\R^{n\times r}$ are orthonormal. Note that $\K-\G_{(r)}^{opt}$ is an SPSD matrix, and thus its trace norm is equal to the trace of this matrix:
	\begin{align}
	\|\K-\G_{(r)}^{opt}\|_*&=\trace\left(\K^{1/2}\big(\eye_{n\times n}- \HH \HH^T\big)\K^{1/2}\right)\nonumber\\
	& \overset{(i)}{=}\trace\left(\K^{1/2}\big(\eye_{n\times n}- \HH \HH^T\big)^2\K^{1/2}\right) \nonumber\\
	& =\|\big(\eye_{n\times n}-\HH\HH^T\big)\K^{1/2}\|_F^2\nonumber \\
	&\overset{(ii)}{=}   \|\K^{1/2} - \M\llbracket\M^T\K^{1/2}\rrbracket_r\|_F^2
	\overset{(iii)}{=}  \min_{\mathbf{T}:\;\rank(\mathbf{T})\leq r} \|\K^{1/2} - \M\mathbf{T}\|_F^2,\label{eq:thm-1-proof}
	\end{align}  
	where (i) follows from $(\eye_{n\times n}- \HH \HH^T\big)^2=(\eye_{n\times n}- \HH \HH^T\big)$, (ii) is based on the observation $\HH\HH^T\K^{1/2}=\M\llbracket\M^T\K^{1/2}\rrbracket_r$, and (iii) is based on Lemma \ref{lemma:best-rank-ONB}. Let us define the function $f(\mathbf{T})\defeq\|\K^{1/2} - \M\mathbf{T}\|_F^2$, and consider the matrix $\mathbf{T}'\in\R^{m\times n}$ with rank no greater than $r$: $\mathbf{T}'=\bigl[ \begin{smallmatrix}\M_r^T\\
	\mathbf{0}_{(m-r)\times n}\end{smallmatrix}\bigr]\K^{1/2}$.
	Then, we see that:
	\begin{align}
	\|\K-\G_{(r)}^{opt}\|_* & \leq  f(\mathbf{T'}) \nonumber\\
	&=  \|\K^{1/2} - \M_r\M_r^T\K^{1/2}\|_F^2\nonumber\\
	&=  \trace\left(\K^{1/2}\big(\eye_{n\times n}- \M_r \M_r^T\big)^2\K^{1/2}\right)  =   \|\K-\G_{(r)}^{nys}\|_*,
	\end{align}
	where we used $\M\mathbf{T}'=\M_r\M_r^T\K^{1/2}$, $(\eye_{n\times n}- \M_r \M_r^T\big)^2=(\eye_{n\times n}- \M_r \M_r^T\big)$, the alternative formulation of $\G_{(r)}^{nys}=\K^{1/2}\M_r\M_r^T\K^{1/2}$, cf.~\eqref{eq:G-nys-alt}, and that $\K-\G_{(r)}^{nys}$ is SPSD. 
\end{proof}

In Example \ref{example1}, we showed that the modified method outperforms the standard Nystr\"om method. The essential structural feature of the example was the presence of a large-magnitude block of the kernel matrix, denoted $\K_{21}$ below. The following remark shows that when this block is zero, the two methods perform the same.
\begin{remark} \label{rmk:same}
	Let $\PP\in\R^{n\times m}$ be the sampling matrix, where $m$ columns of the kernel matrix $\K\in\R^{n\times n}$ are sampled according to any
	distribution. Without loss of generality, the matrices $\CC$ and $\K$ can be permuted as follows:
	\begin{equation}
	\K=\begin{bmatrix}
	\WW & \K_{21}^T\\
	\K_{21} & \K_{22}
	\end{bmatrix},\;\;\CC=\begin{bmatrix}
	\WW\\
	\K_{21}
	\end{bmatrix},
	\end{equation}
	where $\K_{21}\in\R^{(n-m)\times m}$ and $\K_{22}\in\R^{(n-m)\times (n-m)}$. If $\K_{21}=\mathbf{0}_{(n-m)\times m}$, then Nystr\"om via QR decomposition and the standard Nystr\"om method generate the same rank-$r$ approximation of the kernel matrix $\K$, i.e., $\G_{(r)}^{opt}=\G_{(r)}^{nys}$. 
	\begin{proof}
		Given that $\K_{21}=\mathbf{0}_{(n-m)\times m}$, we have $\K^{1/2}=\bigl[\begin{smallmatrix}\WW^{1/2} & \mathbf{0}_{m\times (n-m)}\\
		\mathbf{0}_{(n-m)\times m} & \K_{22}^{1/2} \end{smallmatrix}\bigr]$ and $\D=\K^{1/2}\PP=\bigl[\begin{smallmatrix}	\WW^{1/2}\\
		\mathbf{0}_{(n-m)\times m}\end{smallmatrix}\bigr]$.
		Let us consider the EVD of $\WW=\VV\SIGMA\VV^T$, where $\VV,\SIGMA\in\R^{m\times m}$.
		Then, the left singular vectors of $\D$ have the following form: $\M=\bigl[\begin{smallmatrix}	\VV\\
		\mathbf{0}_{(n-m)\times m}\end{smallmatrix}\bigr]$.
		Thus, we get $\M^T\K^{1/2}=[\VV^T\WW^{1/2},\mathbf{0}_{m\times (n-m)}]$. Note that $\VV^T\WW^{1/2}=\SIGMA^{1/2}\VV^T$, and the best rank-$r$ approximation of $\M^T\K^{1/2}$ can be written as:
		\begin{equation}
		\llbracket\M^T\K^{1/2}\rrbracket_r=\left[\Y\SIGMA_r^{1/2}\VV_r^T,\mathbf{0}_{m\times (n-m)}\right],
		\end{equation}
		where $	\Y\defeq\bigl[\begin{smallmatrix}	\eye_{r\times r}\\
		\mathbf{0}_{(m-r)\times r}\end{smallmatrix}\bigr]\in\R^{m\times r}$.
		Next, we compute the matrix $\mathbf{T}'=\bigl[ \begin{smallmatrix}\M_r^T\\
		\mathbf{0}_{(m-r)\times n}\end{smallmatrix}\bigr]\K^{1/2}$ in the proof of Theorem \ref{thm:nys-qr-sta} by first simplifying $\M_r^T\K^{1/2}=\left[\VV_r^T\WW^{1/2},\mathbf{0}_{r\times (n-m)}\right]$. Since $\VV_r^T\WW^{1/2}=\SIGMA_r^{1/2}\VV_r^T$, we have $\mathbf{T}'=\bigl[\begin{smallmatrix}	\M_r^T\K^{1/2}\\
		\mathbf{0}_{(m-r)\times n}\end{smallmatrix}\bigr]=[\Y\SIGMA_r^{1/2}\VV_r^T,\mathbf{0}_{m\times (n-m)}]$.
		Thus, we get $\mathbf{T}'=\llbracket\M^T\K^{1/2}\rrbracket_r$ and
		this completes the proof.
	\end{proof}
\end{remark}
\begin{remark}\label{thm:remark-frob} In Theorem \ref{thm:nys-qr-sta}, we showed that Nystr\"om via QR decomposition generates improved rank-$r$ approximation of kernel matrices with respect to the trace norm. However, this property is not always satisfied in terms of the Frobenius norm. For example, consider the following $4\times 4$ SPSD  matrix:
	\begin{equation}\label{eq:kernel2}
	\K=\begin{bmatrix}
	1.0 & 0.7 & 0.9 & 0.4\\
	0.7 & 1.0 & 0.6 & 0.6\\
	0.9 & 0.6 & 1.0 & 0.6\\
	0.4 & 0.6 & 0.6 & 1.0
	\end{bmatrix}.
	\end{equation}
	If we sample the first and second column of $\K$ to form $\CC\in\R^{4\times 2}$, i.e., $m=2$, then we get $\|\K - \G_{(1)}^{nys}\|_*=1.3441
	$ and $\|\K - \G_{(1)}^{opt}\|_*=1.3299$. Thus, we have $\|\K - \G_{(1)}^{opt}\|_*\leq \|\K - \G_{(1)}^{nys}\|_*$, as expected by Theorem \ref{thm:nys-qr-sta}. If we compare these two error terms based on the Frobenius norm, then we see that
	$\|\K - \G_{(1)}^{nys}\|_F=0.9397$ and $\|\K - \G_{(1)}^{opt}\|_F=0.9409$. Thus, in this example, the standard Nystr\"om method has slightly better performance in terms of the Frobenius norm. To explain this observation, let us define $g(\mathbf{T})\defeq\| (\K^{1/2} - \M\mathbf{T})(\K^{1/2} - \M\mathbf{T})^T \|_F^2$ for an input matrix $\mathbf{T}\in\R^{m\times n}$ and fixed matrices $\K$ and $\M$. Based on the proof of Theorem \ref{thm:nys-qr-sta}, it is straightforward to show that $\| \K - \G_{(r)}^{nys}\|_F^2=g(\mathbf{T}^{nys})$, where $\mathbf{T}^{nys}=\bigl[ \begin{smallmatrix}\M_r^T\\
	\mathbf{0}_{(m-r)\times n}\end{smallmatrix}\bigr]\K^{1/2}$.
	Also, we have $\|\K-\G_{(r)}^{opt}\|_F^2=g(\mathbf{T}^{opt})$,  
	$\mathbf{T}^{opt}= \llbracket \M^T\K^{1/2}\rrbracket_r$. Additionally, we can express the error term $g(\mathbf{T})$ as the sum of the following three terms:
	\begin{align}
	g(\mathbf{T})=&\| \big(\K^{1/2}-\begin{bmatrix}
	\M & \M^\perp\end{bmatrix} \begin{bmatrix}
	\mathbf{T}\\
	\mathbf{0}
	\end{bmatrix}\big)\big(\K^{1/2}-\begin{bmatrix}
	\M & \M^\perp\end{bmatrix} \begin{bmatrix}
	\mathbf{T}\\
	\mathbf{0}
	\end{bmatrix}\big)^T  \|_F^2\nonumber\\
	=&\underbrace{\| (\M^T\K^{1/2}-\mathbf{T})(\M^T\K^{1/2}-\mathbf{T})^T \|_F^2}_{= g_1(\mathbf{T})}+2\underbrace{\| (\M^T\K^{1/2} - \mathbf{T})( \K^{1/2} \M^\perp) \|_F^2}_{=g_2(\mathbf{T})} + g_3,
	\end{align}
	where $\M^\perp\in\R^{n\times (n-m)}$ is the orthogonal complement of $\M$, we used the unitary invariance of the Frobenius norm, and $g_3=\| (\M^\perp)^T \K \M^\perp \|_F^2$ is independent of $\mathbf{T}$. For the given kernel matrix in \eqref{eq:kernel2},  we have $g_1(\mathbf{T}^{opt})=0.2524$, $g_2(\mathbf{T}^{opt})=0.0274$, $g_1(\mathbf{T}^{nys})=0.2669$, $g_2(\mathbf{T}^{nys})=0.0191$, and $g_3=0.5780$. Since  $g_1(\mathbf{T}^{opt})<g_1(\mathbf{T}^{nys})$ and $g_2(\mathbf{T}^{opt})>g_2(\mathbf{T}^{nys})$ in this example, it is clear that the modified Nystr\"om method does not necessarily lead to more accurate approximations with respect to the Frobenius norm. 
	Even though it possible that the standard Nystr\"om can outperform Nystr\"om via QR decomposition in the Frobenius norm,
	our substantial numerical experiments on real-world data sets in Section \ref{sec:exper} show that this does not happen in practice.
\end{remark}
Next, we present an important theoretical result on the quality of rank-$r$ Nystr\"om approximations when the number of landmark points are increased. Specifically, let us first sample $m_1\geq r$ landmark points from the set of input data points to generate the rank-$r$ approximation using the modified Nystr\"om method, namely $\G_{(r)}^{opt}$. If we sample $(m_2-m_1)\in\mathbb{N}$ additional landmark points to form the new rank-$r$ approximation $\widetilde{\G}_{(r)}^{opt}$ using the total of $m_2$ landmark points, the following result states that $	\|\K-\widetilde{\G}_{(r)}^{opt}\|_*\leq \|\K-\G_{(r)}^{opt}\|_*$. Therefore, increasing the number of distinct landmark points in the modified Nystr\"om method leads to improved rank-$r$ approximation.
\begin{thm}[More landmark points help] \label{thm:nys-qr-std-more}
	Let $\K\in\R^{n\times n}$ be an SPSD kernel matrix, and $r$ be the target rank. Consider a sampling matrix $\PP\in\R^{n\times m_1}$, with $m_1\geq r$, such that $\CC=\K\PP\in\R^{n\times m_1}$ and $\WW=\PP^T\K\PP\in\R^{m_1\times m_1}$. Then, the modified Nystr\"om method generates the rank-$r$ approximation of $\K$ as $\G_{(r)}^{opt}=\llbracket\CC\WW^\dagger\CC^T\rrbracket_r$. Increase the number of distinct landmark points by concatenating the matrix $\PP$ with $\PP^{new}\in\R^{n\times (m_2-m_1)}$ with $m_2>m_1$, i.e., $\widetilde{\PP}=[\PP,\PP^{new}]\in\R^{n\times m_2}$. The resulting matrix $\widetilde{\PP}$ can be used to form  $\widetilde{\CC}=\K\widetilde{\PP}\in\R^{n\times m_2}$ and $\widetilde{\WW}=\widetilde{\PP}^T\K\widetilde{\PP}\in\R^{m_2\times m_2}$, and the modified Nystr\"om method generates $\widetilde{\G}_{(r)}^{opt}=\llbracket\widetilde{\CC}\widetilde{\WW}^\dagger\widetilde{\CC}^T\rrbracket_r$. Then this new approximation is better in the sense that: 
	\begin{equation}
	\|\K-\widetilde{\G}_{(r)}^{opt}\|_*\leq \|\K-\G_{(r)}^{opt}\|_*.
	\end{equation}
\end{thm}
\begin{proof}
	Let $\M\in\R^{n\times m_1}$ and $\widetilde{\M}\in\R^{n\times m_2}$  be the left singular vectors of $\K^{1/2}\PP\in\R^{n\times m_1}$ and $\K^{1/2}\widetilde{\PP}=[\K^{1/2}\PP,\K^{1/2}\PP^{new}]\in\R^{n\times m_2}$, respectively. Then, we get: 
	\begin{equation}
	\|\K- \widetilde{\G}_{(r)}^{opt}\|_*= \min_{\widetilde{\mathbf{T}}:\;\rank(\widetilde{\mathbf{T}})\leq r} \|\K^{1/2} - \widetilde{\M} \widetilde{\mathbf{T}}\|_F^2\nonumber \\
	\leq   \min_{\mathbf{T}:\;\rank(\mathbf{T})\leq r} \|\K^{1/2} - \M \mathbf{T}\|_F^2\nonumber\\
	=  \|\K- \G_{(r)}^{opt}\|_*
	\end{equation}
	where both equalities follow from \eqref{eq:thm-1-proof} and 
	the inequality follows from the fact that $\text{Range}(\M)\subset\text{Range}(\widetilde{\M})$. 
\end{proof}

\begin{remark} \label{rmk:1}
	Theorem \ref{thm:nys-qr-std-more} is not true for the standard Nystr\"om method. 
	Consider the kernel matrix from Example \ref{example1}. By sampling the first two columns,
	the standard Nystr\"om method gave relative errors of $0.99$ in both the trace and Frobenius norms.
	Had we sampled just the first column, the standard Nystr\"om method would have $0.01$ relative error in these norms, meaning that adding additional landmark points leads to a worse approximation.  See also Remark \ref{rmk:2} for experiments.
\end{remark}
\section{Extension to Out-of-Sample Landmark Points}\label{sec:thm-outofsample}
The main component in our theoretical results is the existence of $\PP\in\R^{n\times m}$ such that $\CC$ and $\WW$ can be written as: $\CC=\K\PP$ and $\WW=\PP^T\K\PP$. As mentioned, this assumption holds true if the landmark points are selected 
(randomly or arbitrarily)
from the set of input data points, since then $\PP$ is a sampling matrix consisting of columns of the identity matrix. However, some recent selection techniques utilize out-of-sample extensions of the input data to improve the accuracy of the Nystr\"om method, e.g., centroids found from K-means clustering. In this case, the matrix $\CC\in\R^{n\times m}$, where $C_{ij}=\kappa(\x_i,\z_j)$, does not necessarily contain the columns of $\K$. Thus, we cannot hope for a sampling matrix $\PP$ that satisfies $\CC=\K\PP$. In this section, we present two techniques to show that our theoretical results hold for the case of out-of-sample landmark points.

\subsection{Approach 1: Full-Rank Kernel Matrices}
We show that, under certain conditions, our theoretical results in Section \ref{sec:theory} are applicable to the case of out-of-sample landmark points. To be formal, consider a set of $n$ distinct data points in $\R^p$, i.e., $\X=[\x_1,\ldots,\x_n]\in\R^{p\times n}$, and the Gaussian kernel of the form $\kappa(\x_i,\x_j)=\exp(-\|\x_i-\x_j\|_2^2/c)$, $c>0$, which leads to the kernel matrix $\K\in\R^{n\times n}$. Let $\Z=[\z_1,\ldots,\z_m]\in\R^{p\times m}$ be $m$ cluster centroids from K-means clustering on $\x_1,\ldots,\x_n$, and we form $\CC\in\R^{n\times m}$ and $\WW\in\R^{m\times m}$ with $C_{ij}=\kappa(\x_i,\z_j)$ and $W_{ij}=\kappa(\z_i,\z_j)$. 

Based on \cite[Theorem 2.18]{LearningWithKernels}, the kernel matrix $\K$ defined on a set of $n$ distinct data points using the Gaussian kernel function has full rank. Thus by defining $\PP=\K^{-1}\CC\in\R^{n\times m}$, we can write $\CC=\K\PP$, but because this $\PP$ is not a sampling matrix, it does not follow that $\WW=\PP^T\K\PP$, so our aim is to show that $\WW\approx\PP^T\K\PP$. Let us consider the \emph{empirical} kernel map $\Phi_e$, defined on the set of input data points:
\begin{equation}
\Phi_{e}(\z): \z\mapsto \K^{-1/2}\left[\kappa(\x_1,\z),\ldots,\kappa(\x_n,\z)\right]^T\in\R^n.\label{eq:emp-kernel-map}
\end{equation}
This map approximates the kernel-induced map $\Phi$ for out-of-sample data points $\z_1,\ldots,\z_m$ such that $\langle\Phi_e(\z_i),\Phi_e(\z_j)\rangle\approx \langle \Phi(\z_i),\Phi(\z_j)\rangle=\kappa(\z_i,\z_j)$~\cite{LearningWithKernels}. 
Since $\PP^T=\CC^T\K^{-1}$ and the $j$-th column of $\CC$ is $[\kappa(\x_1,\z_j),\ldots,\kappa(\x_n,\z_j)]^T$, we have: 
\begin{align}
\PP^T\K\PP&=\big(\CC^T\K^{-1/2}\big)\big(\K^{-1/2}\CC\big)\nonumber\\
& = \left[\Phi_e(\z_1),\ldots,\Phi_e(\z_m)\right]^T\left[\Phi_e(\z_1),\ldots,\Phi_e(\z_m)\right]\defeq\WW_e\in\R^{m\times m}.
\end{align}
Therefore, if we use out-of-sample landmark points with the Gaussian kernel function, there exists a matrix $\PP$ that satisfies $\CC=\K\PP$
and $\WW=\WW_e+\EE=\PP^T\K\PP+\EE$, where $\EE\in\R^{m\times m}$ represents the approximation error.
It is known that when the relative amount of error is small (e.g., with respect to the spectral norm), $\WW$ and $\WW_e$ are close to one another and their eigenvalues and eigenvectors are perturbed proportional to the relative error \cite{mathias1998relative,dopico2000weyl}. However, in this work, our goal is to prove that the approximations $\CC\WW^\dagger\CC^T$ and $\CC\WW_e^\dagger\CC^T$ are close to one another when the relative amount of error is small. To demonstrate the importance of this result, note that for any invertible matrices $\mathbf{M}$ and $\mathbf{M}'$, we have the identity $\mathbf{M}'^{-1}-\mathbf{M}^{-1}=-\mathbf{M}'^{-1}(\mathbf{M}'-\mathbf{M})\mathbf{M}^{-1}$. Thus, the small norm of $\mathbf{M}'-\mathbf{M}$ cannot be directly used to conclude $\mathbf{M}'^{-1}$ and $\mathbf{M}^{-1}$ are close to one another. In the following, we present an error bound for the difference between the Nystr\"om approximations, i.e., $\CC\WW^\dagger\CC^T-\CC\WW_e^\dagger\CC^T$, in terms of the relative amount of error caused by the empirical kernel map. 
\begin{thm}\label{thm:out-of-sample} Consider a set of $n$ distinct data points $\X=[\x_1,\ldots,\x_n]\in\R^{p\times n}$, and the Gaussian kernel function of the form $\kappa(\x_i,\x_j)=\exp(-\|\x_i-\x_j\|_2^2/c)$, $c>0$, which leads to the kernel matrix $\K\in\R^{n\times n}$. Let $\Z=[\z_1,\ldots,\z_m]\in\R^{p\times m}$ be arbitrary (e.g., $m$ distinct cluster centroids from K-means clustering on $\x_1,\ldots,\x_n$), and we form $\CC\in\R^{n\times m}$ and $\WW\in\R^{m\times m}$ with $C_{ij}=\kappa(\x_i,\z_j)$ and $W_{ij}=\kappa(\z_i,\z_j)$. Then, there exists a matrix $\PP\in\R^{n\times m}$ such that $\CC=\K\PP$ and $\WW=\WW_e+\EE$, where $\WW_e=\PP^T\K\PP$ and $\EE\in\R^{m\times m}$ represents the approximation error of the empirical kernel map defined in \eqref{eq:emp-kernel-map}. Assuming that $\eta\defeq\|\WW_e^{-1/2}\EE\WW_e^{-1/2}\|_2<1$ for the positive definite matrix $\WW_e$, then:
	\begin{equation}
	\frac{\|\CC\WW^{\dagger}\CC^T - \CC\WW_e^\dagger \CC^T\|_2}{\|\K\|_2} \leq \frac{\eta}{1-\eta}.
	\end{equation}
\end{thm}
\newcommand{\OO}{\mathbf{O}}
\begin{proof} The kernel matrix $\K$ using the Gaussian kernel function has full rank. Thus, there exists a matrix $\PP$ such that $\CC=\K\PP$ and the empirical kernel map $\Phi_e$ can be defined as in \eqref{eq:emp-kernel-map}. Recall the SVD of $\D=\K^{1/2}\PP=\M\SSS\NN^T$, and the EVD of $\WW_e=\PP^T\K\PP=\NN\SSS^2\NN^T$. Let $\WW=\widetilde{\NN}\widetilde{\SSS}^2\widetilde{\NN}^T$ be the EVD of $\WW$ (since $\WW$ is also SPSD). Moreover, let us define:
	\begin{equation}
	\widetilde{\EE}\defeq \SSS \NN^T \widetilde{\NN} \widetilde{\SSS}^{-2} \widetilde{\NN}^T\NN\SSS - \eye_{m\times m}\in\R^{m\times m}.\label{eq:error-perturb}
	\end{equation} 
	If we have $\EE=\mathbf{0}_{m\times m}$, i.e., the approximate kernel map $\Phi_e$ is equal to the kernel-induced map, then $\widetilde{\NN}=\NN$, $\widetilde{\SSS}=\SSS$, and $\widetilde{\EE}=\mathbf{0}_{m\times m}$. Next, we find an upper bound for $\|\widetilde{\EE}\|_2$ in terms of the relative error $\eta$. Consider the EVD of $\WW=\WW_e+\EE$:
	\begin{equation}
	\widetilde{\NN}\widetilde{\SSS}^2\widetilde{\NN}^T = \NN\SSS^2\NN^T+\EE
	= \NN\SSS\left(\eye_{m\times m}+\mathbf{O}\right)\SSS\NN^T, \label{eq:perturb1}
	\end{equation}
	where $\OO\defeq\SSS^{-1}\NN^T\EE\NN\SSS^{-1}\in\R^{m\times m}$ is a symmetric matrix. Note that $\|\mathbf{O}\|_2=\|\NN\OO\NN^T\|_2=\eta$, because of the unitary invariance of the spectral norm. If we multiply \eqref{eq:perturb1} on the left by $\NN^T$ and on the right by $\widetilde{\NN}$, we get:
	\begin{equation}
	\NN^T\widetilde{\NN}\widetilde{\SSS}^2 = \SSS\left(\eye_{m\times m}+\mathbf{O}\right)\SSS\NN^T\widetilde{\NN}. \label{eq:perturb2}
	\end{equation}
	Next, we multiply \eqref{eq:perturb2} on the left by $\widetilde{\NN}^T\NN$, and we see that:
	\begin{equation}
	\widetilde{\SSS}^2= \widetilde{\NN}^T\NN \SSS^2 \NN^T\widetilde{\NN}+\widetilde{\NN}^T\NN\SSS\mathbf{O} \SSS \NN^T\widetilde{\NN}.\label{eq:perturb3}
	\end{equation}
	Finally, we multiply \eqref{eq:perturb3} on the left and right by $\widetilde{\SSS}^{-1}$:
	\begin{equation}
	\eye_{m\times m}=\big(\widetilde{\SSS}^{-1}\widetilde{\NN}^T\NN \SSS\big)\big(\eye_{m\times m} + \mathbf{O}\big)\big(\SSS \NN^T\widetilde{\NN}\widetilde{\SSS}^{-1}\big).
	\end{equation}
	Thus, we observe that $\SSS \NN^T\widetilde{\NN}\widetilde{\SSS}^{-1}=(\eye_{m\times m}+\mathbf{O})^{-1/2}\mathbf{T}$, where $\mathbf{T}\in\R^{m\times m}$ is an orthogonal matrix. Thus, we have:
	\begin{equation}
	\widetilde{\EE}=\big(\eye_{m\times m} + \mathbf{O}\big)^{-1} - \eye_{m\times m}.\label{eq:perturb4}
	\end{equation}
	To find an upper bound for the spectral norm of  $\widetilde{\EE}$, we simplify \eqref{eq:perturb4} by using the Neumann series $(\eye_{m\times m} + \mathbf{O})^{-1} = \sum_{n=0}^{\infty}(-1)^n\mathbf{O}^{n}$, where $\|\mathbf{O}\|_2=\eta<1$ by assumption. Hence, we get the following upper bound for the spectral norm of $\widetilde{\EE}=\sum_{n=1}^{\infty}(-1)^n\mathbf{O}^n$:
	\begin{equation}
	\|\widetilde{\EE}\|_2\leq\sum_{n=1}^{\infty} \|\mathbf{O}^n\|_2\leq\sum_{n=1}^{\infty} \|\mathbf{O}\|_2^n\leq\sum_{n=1}^{\infty} \eta^n=\frac{\eta}{1-\eta},
	\end{equation}
	where we have used the convergence of the Neumann series and the continuity of norms in the first inequality and  the submultiplicativity property of the spectral norm in the second. 
	To finish, we relate
	the difference between the two Nystr\"om approximations $\CC\WW^\dagger\CC^T$ and $\CC\WW_e^\dagger\CC^T$ to the norm of $\widetilde{\EE}$:
	\begin{equation}
	\CC\big(\WW^\dagger\big)^{1/2}=\K^{1/2}\D\Big(\widetilde{\NN}\widetilde{\SSS}^{-1}\widetilde{\NN}^T\Big)=\K^{1/2} \M \SSS \NN^T \widetilde{\NN} \widetilde{\SSS}^{-1} \widetilde{\NN}^T.
	\end{equation}
	Then, given the definition of $\widetilde{\EE}$ in \eqref{eq:error-perturb}, we observe that:
	\begin{equation}
	\CC\WW^{\dagger}\CC^T - \CC\WW_e^\dagger \CC^T = \K^{1/2}\M\widetilde{\EE} \M^T\K^{1/2}.
	\end{equation}
	Thus, using the submultiplicativity property of the spectral norm, we have:
	\begin{equation}
	\|\CC\WW^{\dagger}\CC^T - \CC\WW_e^\dagger \CC^T\|_2\leq \|\K^{1/2}\|_2^2\|\M\|_2^2\|\widetilde{\EE}\|_2.
	\end{equation}
	Note that $\|\K^{1/2}\|_2^2=\|\K\|_2$, and $\|\M\|_2=1$ since $\M$ has orthonormal columns.
\end{proof}
To gain some intuition for Theorem \ref{thm:out-of-sample},  we present a numerical experiment on the  \dataset{pendigits} data set ($p=16$ and $n=10,\!992$)  used in Section \ref{sec:exper}. Here, the Gaussian kernel function is employed with the parameter $c$ chosen as the averaged squared distances between all the data points and sample mean. The standard K-means clustering algorithm is performed on the input data points to select the landmark points $\z_1,\ldots,\z_m$ for various values of $m=2,\ldots,10$. For each value of $m$, we form two matrices $\CC$ and $\WW$. Also, we compute $\WW_e=\PP^T\K\PP$ and $\eta=\|\WW_e^{-1/2}\EE\WW_e^{-1/2}\|_2$, where $\PP=\K^{-1}\CC$ and $\EE=\WW-\WW_e$; calculating $\WW_e$ is impractical for larger data sets and we do so only to support our theorem.  Figure \ref{fig:bound} reports the mean of $\|\CC\WW^{\dagger}\CC^T - \CC\WW_e^\dagger \CC^T\|_2/\|\K\|_2$ over $50$ trials for varying number of landmark points. The figure also plots the mean of our theoretical bound in Theorem \ref{thm:out-of-sample}, i.e., $\eta/(1-\eta)$. We observe that $\CC\WW^\dagger\CC^T$ and $\CC\WW_e^\dagger\CC^T$ provide very similar Nystr\"om approximations of $\K$, such that the relative error with respect to the spectral norm is less than $0.002$. Furthermore, it is clear that Theorem \ref{thm:out-of-sample}  provides a meaningful upper bound for the relative error of the Nystr\"om approximations.
\begin{figure}[tbhp]
	\centering
	\includegraphics[width=0.6\textwidth]{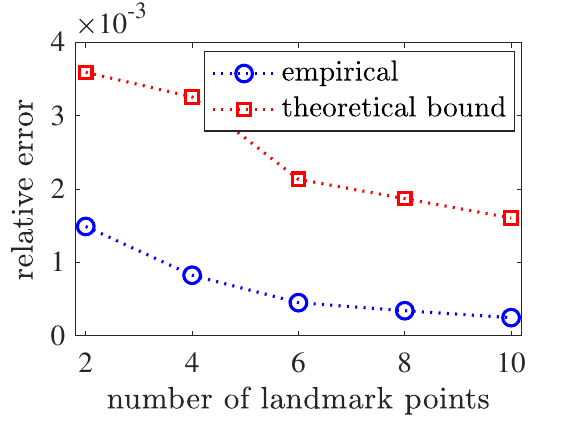}
	\caption{Mean of the relative error and the theoretical error bound $\eta/(1-\eta)$.}
	\label{fig:bound}
\end{figure}

Based on Theorem \ref{thm:out-of-sample}, the closeness of $\CC\WW^\dagger\CC^T$ and $\CC\WW_e^\dagger\CC^T$ with respect to the spectral norm is a function of the quantity $\eta$. Note that $\eta$ measures the relative amount of perturbation of the eigenvalues and eigenvectors of $\WW_e$, and we have $\eta=\|\WW_e^{-1/2}\EE\WW_e^{-1/2}\|_2\leq \|\WW_e^{-1}\EE\|_2\leq \|\WW_e^{-1}\|_2\|\EE\|_2$ \cite[Lemma 2.2]{dopico2000weyl}. Therefore, when $\|\EE\|_2$ is small, $\CC\WW^\dagger\CC^T$ and $\CC\WW_e^\dagger\CC^T$ lead to similar low-rank approximations of the kernel matrix. In particular, as $\|\EE\|_2$ goes to zero, $\CC\WW_e^\dagger\CC^T$ converges to  $\CC\WW^\dagger\CC^T$. 
Hence, we expect our theoretical results on the rank-$r$ Nystr\"om approximations to be valid for the case out-of-sample landmark points for small values of $\eta$.  

\subsection{Approach 2: Clustered Nystr\"om} We present an alternative to our previous approach that does not require full-rank kernel matrices. This technique is motivated by the clustered Nystr\"om method of Zhang et al.~\cite{zhang2008improved,zhang2010clusteredNys}, which shows that the Nystr\"om error can be bounded
by the quantization error of the input data using landmark points. Assume that the kernel function $\kappa$ satisfies: $(\kappa(\mathbf{a},\mathbf{b})-\kappa(\mathbf{c},\mathbf{d}))^2\leq \eta (\|\mathbf{a}-\mathbf{c}\|_2^2+\|\mathbf{b}-\mathbf{d}\|_2^2)$, where $\eta$ is a constant depending on $\kappa$. It is shown that for a number of widely used kernel functions, such as Gaussian and polynomial kernels, this property is satisfied. Then, the Nystr\"om approximation error is upper bounded: $\|\K-\CC\WW^\dagger\CC^T\|_F\leq \eta_1\sqrt{\phi}+\eta_2\phi$, where $\phi\defeq1/n\cdot\sum_{i=1}^{n}\|\x_i - \mu(\x_i)\|_2^2$, $\mu(\x_i)\in\R^p$ is the closest landmark point to each data point $\x_i$, and $\eta_1$, $\eta_2$ are two constants.

Given out-of-sample landmark points, this approach approximates the landmark points with in-sample points, and pays just a factor of $2$ on the quantization error. To prove this, without loss of generality and to simplify notation, we consider a single cluster since we can work cluster-by-cluster. If the cluster assignments change after selecting the new landmark points, that can only reduce the quantization error further. Thus, let $\mathbf{s}$ be the (possibly out-of-sample) center of mass for the points $\x_i, i = 1,\ldots, n$, i.e., $\mathbf{s} = 1/n\cdot \sum_{i=1}^n \x_i$. As observed in \cite[Lemma 2.1]{kmeans_plusplus}, for any vector $\mathbf{t}\in \R^p$:
\begin{equation}\label{eq:lemma212}
\frac{1}{n}\sum_{i=1}^n \|\x_i - \mathbf{t} \|_2^2 = \frac{1}{n}\sum_{i=1}^n \|\x_i - \mathbf{s} \|_2^2 + \|\mathbf{t}-\mathbf{s}\|_2^2.
\end{equation}
We will choose a new in-sample cluster center: $\mathbf{t} = \x_j\;\text{for}\; j\in \argmin_{i\in \{1,\ldots,n\}}\, \|\x_i - \mathbf{s}\|_2^2$.
Let $\widehat{\phi}$ denote the quantization error with this new in-sample center, and $\phi$ be the original quantization error using $\mathbf{s}$ as the cluster center. Then
\begin{align}
\widehat{\phi} \defeq \frac{1}{n}\sum_{i=1}^n \|\x_i - \mathbf{t} \|_2^2 &= \frac{1}{n}\sum_{i=1}^n \|\x_i - \mathbf{s} \|_2^2 + \|\mathbf{t}-\mathbf{s}\|_2^2 \nonumber\\
&\le \frac{1}{n}\sum_{i=1}^n \|\x_i - \mathbf{s} \|_2^2 +  \sum_{i=1}^n\frac{1}{n}\|\x_i - \mathbf{s}\|_2^2= 2 \phi,
\end{align}
where the first line follows from \eqref{eq:lemma212}.
Note that the computational cost is small, since choosing an in-sample cluster center is $\order(np)$.

\section{Experimental Results}\label{sec:exper}
We present experimental results on the fixed-rank approximation of kernel matrices using the standard and modified Nystr\"om methods, and show evidence to (1) corroborate our theory of Section \ref{sec:theory}, (2) suggest that our theory of Section \ref{sec:thm-outofsample} still holds even if the assumptions are not exactly satisfied (e.g., out-of-sample landmark points), and (3) highlight the benefits of the modified method.
In order to illustrate the effectiveness of modified Nystr\"om, we compare its  accuracy to that of the standard Nystr\"om method for the target rank $r=2$ and varying number of landmark points $m=r,\ldots,5r$. To provide a baseline for the comparison, we report the accuracy of the best rank-$2$ approximation  obtained via the eigenvalue decomposition (EVD), which requires the computation and storage of full kernel matrices and hence is impractical for very large data sets.

\subsection{Fixed-Rank Approximation Error}\label{sec:exp-1}Experiments are conducted on four data sets from the LIBSVM archive~\cite{CC01a}: (1) 	\dataset{pendigits} ($p=16$ and $n=10,\!992$); (2) \dataset{satimage} ($p=36$ and $n=6,\!435$); (3) \dataset{w6a} ($p=300$ and $n=13,\!267$); and (4) \dataset{E2006-tfidf} ($p=150,\!360$ and $n=3,\!000$). In all experiments,  the Gaussian kernel $\kappa\left(\x_i,\x_j\right)=\exp\left(-\|\x_i-\x_j\|_2^2/c\right)$ is used with the parameter $c$ chosen as the averaged squared distances between all the data points and sample mean. We consider two landmark selection techniques: 
(1) uniform sampling, where $m$ landmark points are selected uniformly at random without replacement from $n$ data points; and (2) out-of-sample landmark points obtained via K-means clustering on the original data set, as in \cite{zhang2010clusteredNys}. To perform the K-means clustering algorithm, we use MATLAB's built-in function \texttt{kmeans} and the maximum number of iterations is set to $10$. 
A MATLAB implementation of modified and standard Nystr\"om is available at  \url{https://github.com/pourkamali/RandomizedClusteredNystrom}.

We measure the quality of fixed-rank approximations using the relative error with respect to the trace norm, i.e., $\|\K-\G_{(r)}^{nys}\|_*/\|\K\|_*$ vs.~$\|\K-\G_{(r)}^{opt}\|_*/\|\K\|_*$,
and thus in this metric, no method can ever outperform the EVD baseline.
In Figure \ref{fig:err-trace}, the mean and standard deviation of relative error over $50$ trials are reported for all four data sets and varying number of landmark points $m=r,\ldots,5r$. Modified Nystr\"om and the standard Nystr\"om method have identical performance for $m=r$ because no rank restriction step is involved in this case.
As the number of landmark points increases beyond the rank parameter $r=2$,
Nystr\"om via QR decomposition always generates better rank-$2$ approximations of the kernel matrix
than the standard method does. 
This observation is consistent with Theorem \ref{thm:nys-qr-sta}, which states that for the same landmark points, $\|\K-\G_{(r)}^{opt}\|_*\leq\|\K-\G_{(r)}^{nys}\|_*$ (one can divide both sides by constant $\|\K\|_*$ for the relative error).
Even with out-of-sample landmark points, the modified method can be drastically better. 
In the right plot of Figure \ref{fig:err_trace_2}, the modified method achieves a mean error of $0.47$ (compared to the $0.45$ EVD baseline) from just $m=2r$ landmark points, while the standard method has a higher mean error ($0.50$) even using $m=5r$ landmark points. This exemplifies the importance and effectiveness of the precise rank restriction step in the Nystr\"om via QR decomposition method.
\begin{figure}[tbhp]
	\centering
	\subfloat[\dataset{pendigits}]{
		\includegraphics[width=0.75\textwidth]{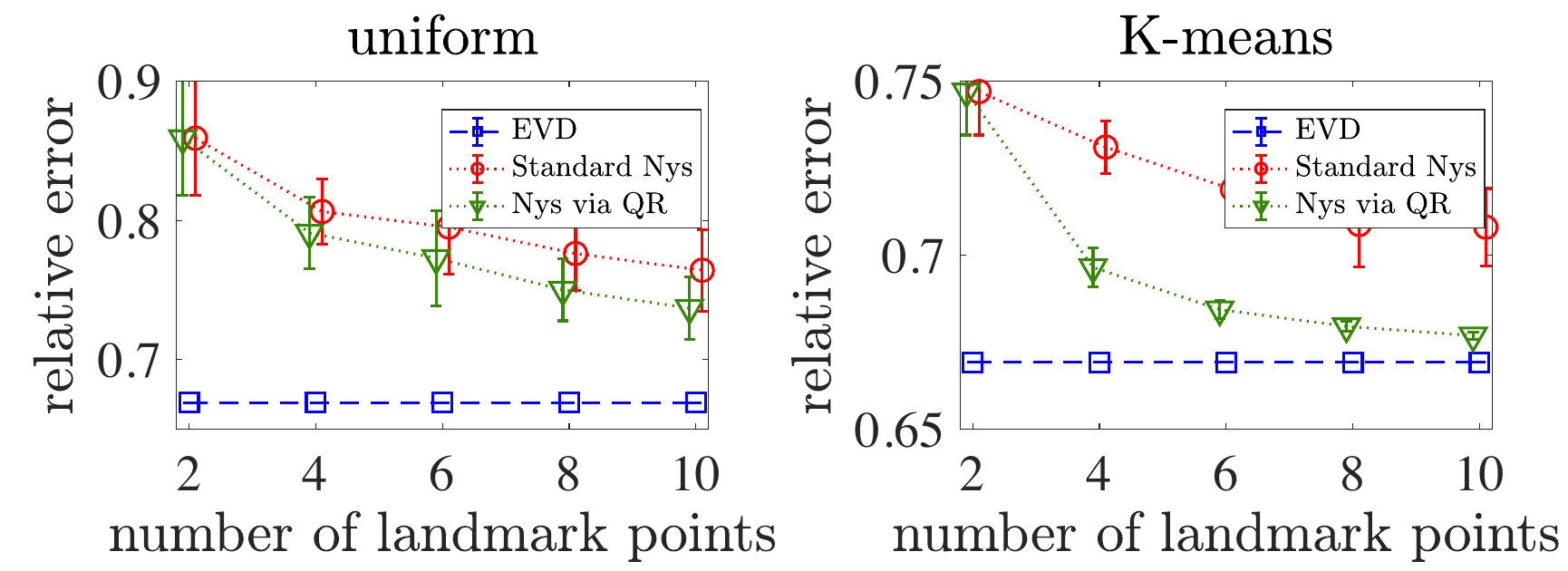}
		\label{fig:err_trace_1}
	}
	
	\subfloat[\dataset{satimage}]{
		\includegraphics[width=0.75\textwidth]{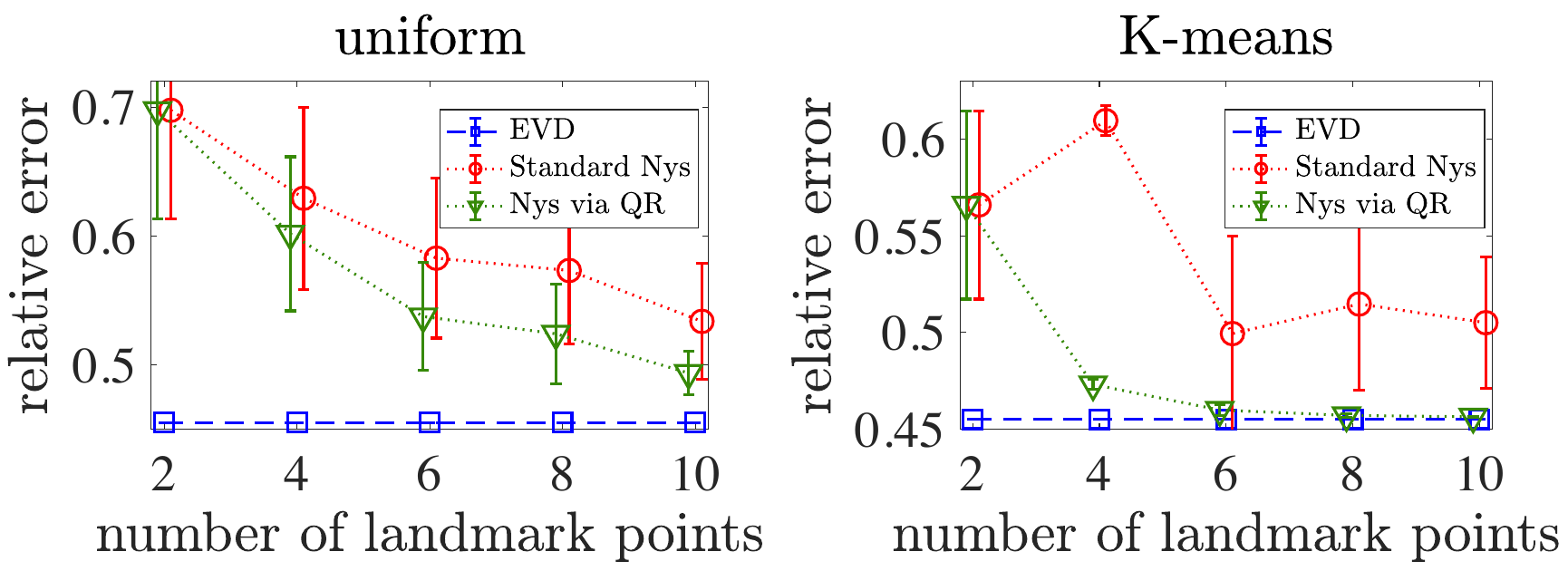}\label{fig:err_trace_2}
	}\\
	
	\centering
	\subfloat[\dataset{w6a}]{
		\includegraphics[width=0.75\textwidth]{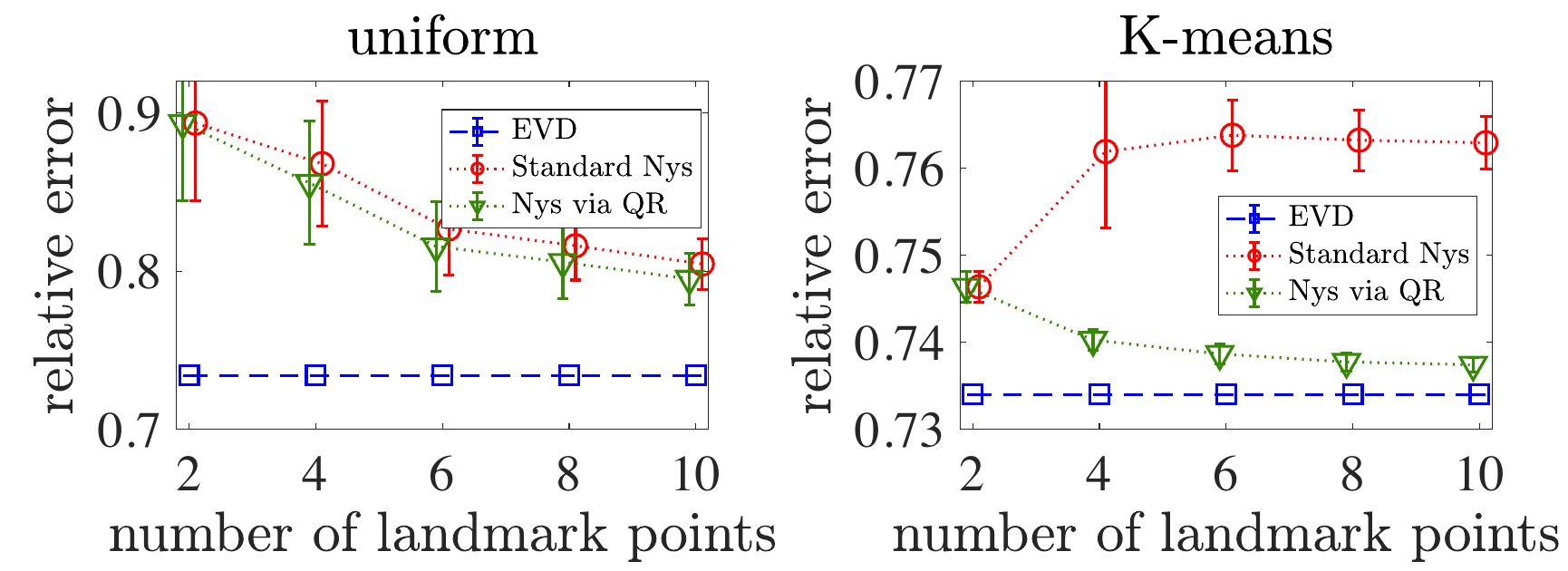}\label{fig:err_trace_3}
	}
	
	\subfloat[\dataset{E2006-tfidf}]{
		\includegraphics[width=0.75\textwidth]{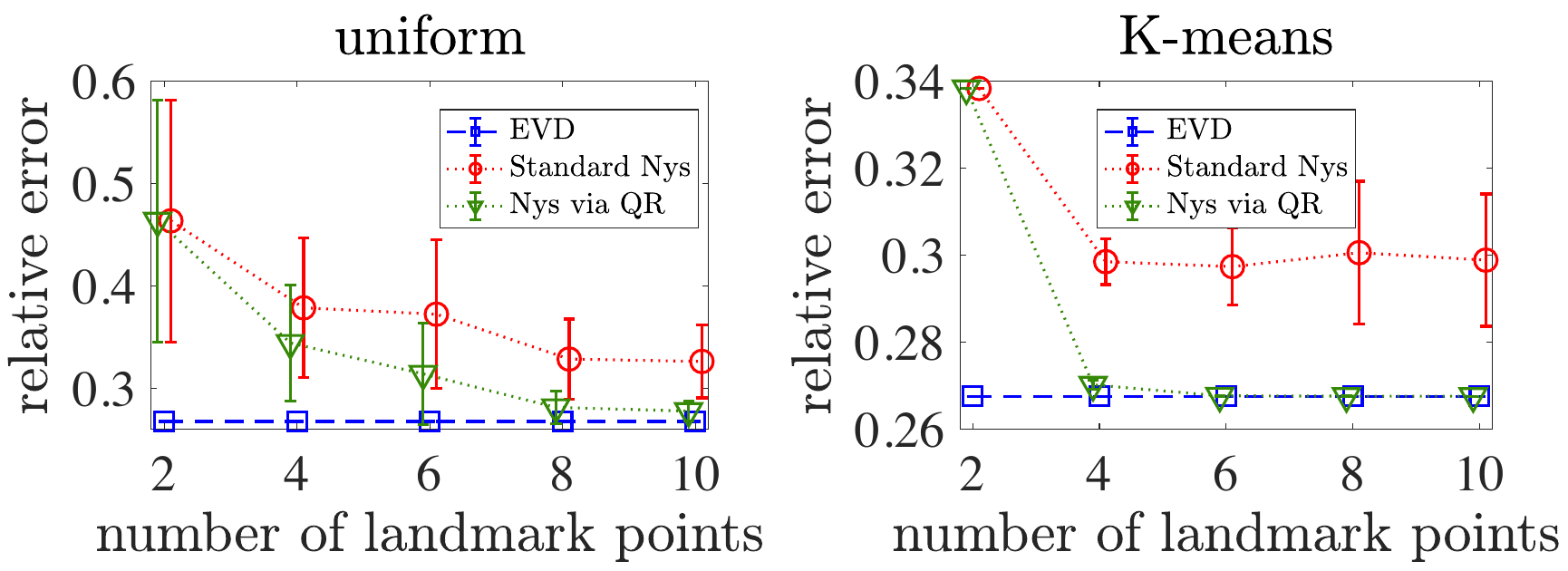}\label{fig:err_trace_4}
	}
	\caption{Mean and standard deviation of the relative error with respect to the trace norm.  For each case, the left plot uses in-sample landmark points while the right plot uses out-of-sample points.
	}
	\label{fig:err-trace}
\end{figure}

Figure \ref{fig:err-frob} is generated the same as Figure \ref{fig:err-trace} but the error is reported in the Frobenius norm. The pattern of behavior is very similar to that  in Figure \ref{fig:err-trace} even though we lack theoretical guarantees. In fact, neither method dominates the other for all kernel matrices (cf.~the adversarial example in Remark \ref{thm:remark-frob}), but in these practical data sets, the modified method always performs better. 

\begin{figure}[tbhp]
	\centering
	\subfloat[\dataset{pendigits}]{
		\includegraphics[width=0.75\textwidth]{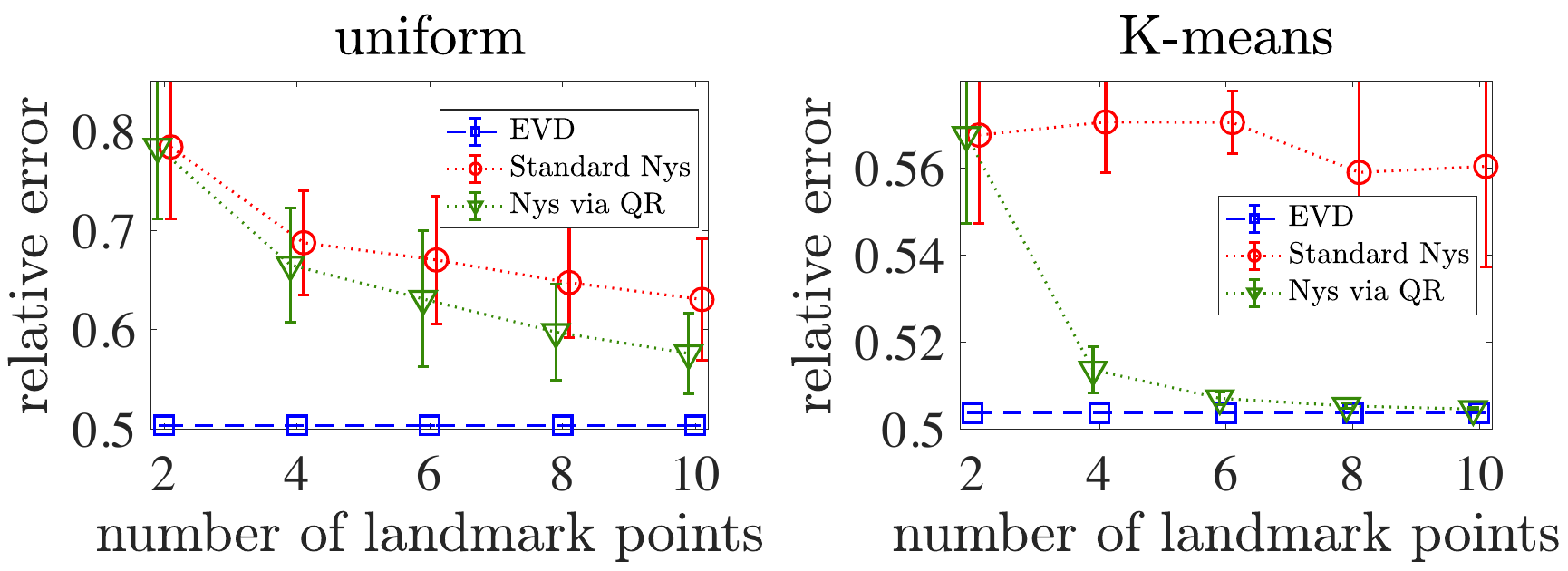}
		\label{fig:err_frob_1}
		
	}
	
	\subfloat[\dataset{satimage}]{
		\includegraphics[width=0.75\textwidth]{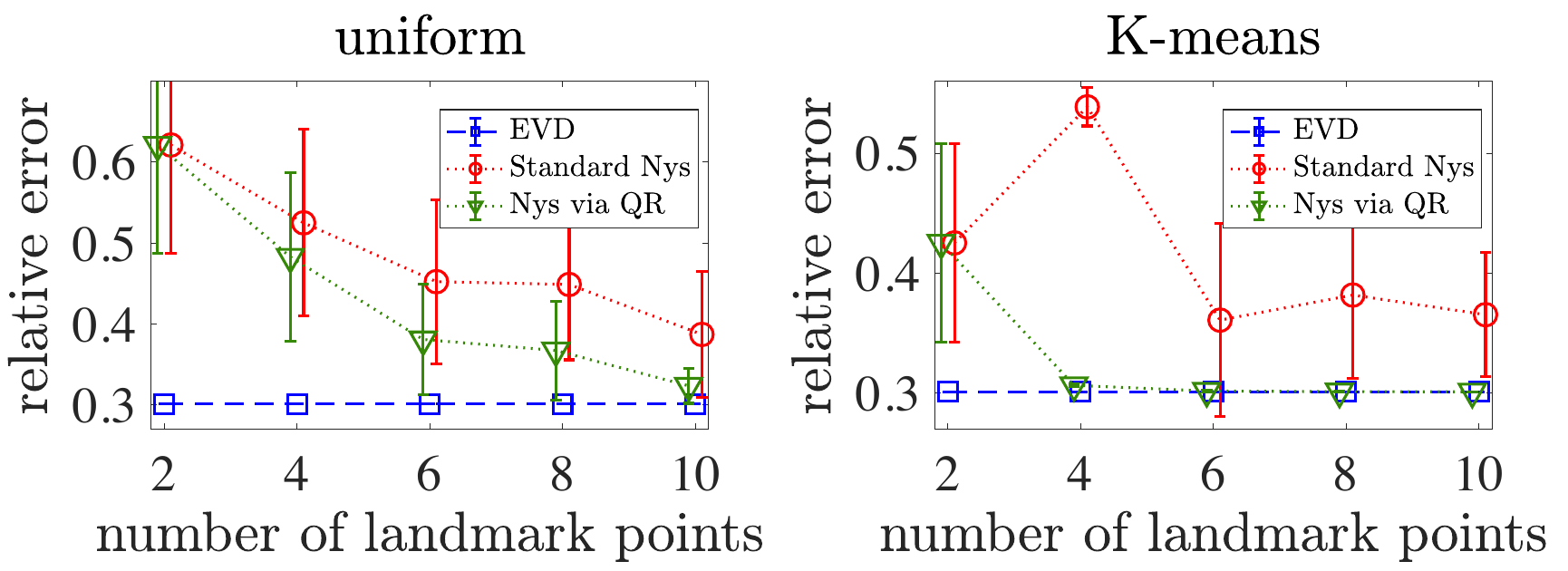}\label{fig:err_frob_2}
	}\\
	
	\centering
	\subfloat[\dataset{w6a}]{
		\includegraphics[width=0.75\textwidth]{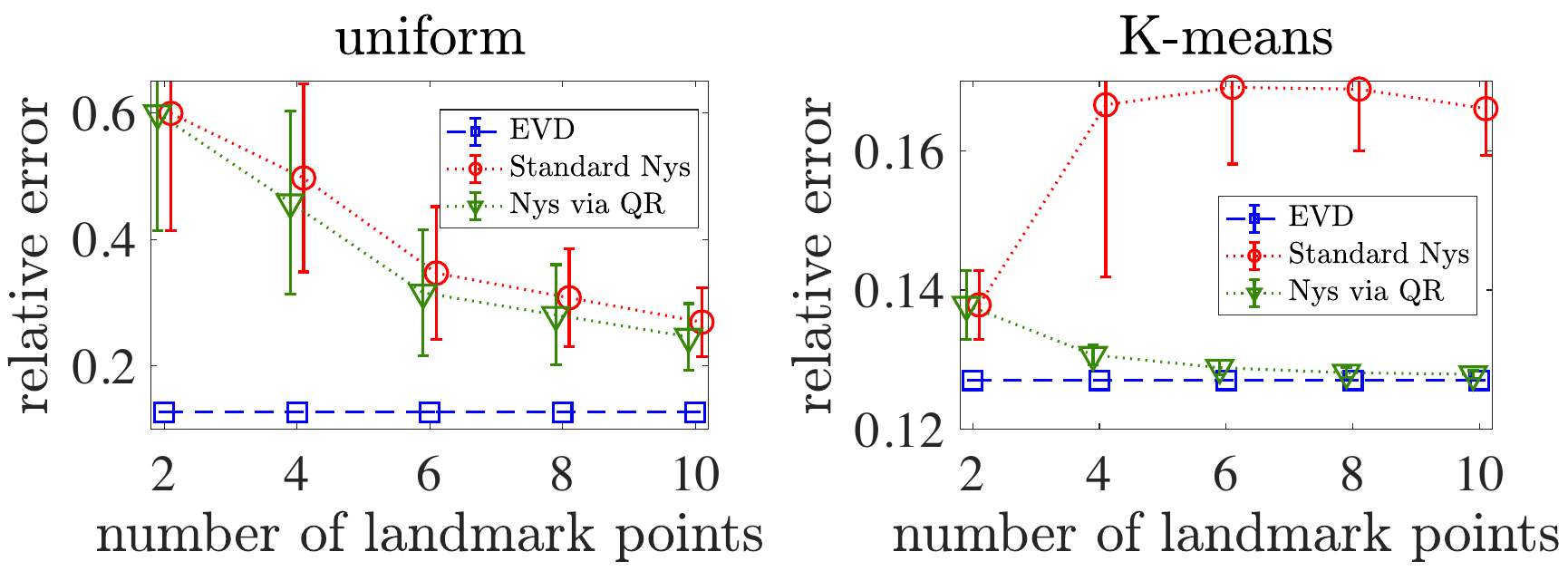}\label{fig:err_frob_3}
	}
	
	\subfloat[\dataset{E2006-tfidf}]{
		\includegraphics[width=0.75\textwidth]{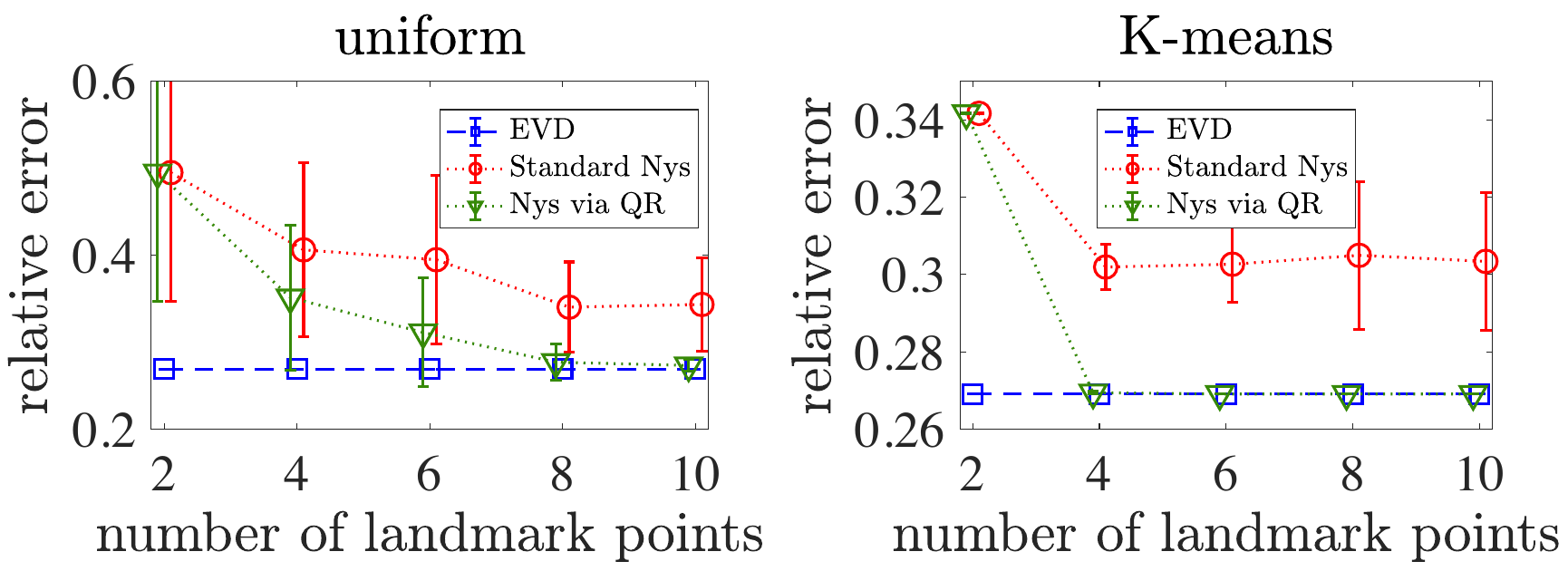}\label{fig:err_frob_4}
	}
	\caption{Mean and standard deviation of the relative error with respect to the Frobenius norm.}
	\label{fig:err-frob}
\end{figure} 

\begin{remark} \label{rmk:2}
	Remark \ref{rmk:1} showed that in both the trace and Frobenius norms, the standard Nystr\"om method can perform worse when we sample additional landmark points. Figure \ref{fig:err-trace} and Figure \ref{fig:err-frob} show that a similar effect happens with the standard Nystr\"om method when we use out-of-sample landmark points selected via K-means (in this case, as we increase $m$, we do not necessarily include the landmark points selected for smaller $m$). For example, according to Figure \ref{fig:err_trace_2} (right), the mean relative error of standard Nystr\"om is increased from $0.56$ to $0.61$ when we increase from $m=2$ to $m=4$ landmark points selected via K-means centroids.
\end{remark}
This counter-intuitive effect of decreased accuracy even with more landmark points (Remark \ref{rmk:1} and Remark \ref{rmk:2}) is due to the sub-optimal restriction procedure of standard Nystr\"om. Theorem \ref{thm:nys-qr-std-more} proves that the \emph{modified} Nystr\"om method does not suffer from the same effect in terms of the trace norm and in-sample landmark points, and Figure \ref{fig:err-trace} and Figure \ref{fig:err-frob} do not show any evidence of this effect even if we switch to the Frobenius norm or consider out-of-sample landmark points.

We also demonstrate the efficiency of Nystr\"om via QR decomposition by plotting the averaged running time on a logarithmic scale for \dataset{E2006-tfidf}. The running time results are omitted for the remaining data sets because the average running time was less than one second. Figure \ref{fig:runtime} shows that the dominant computational cost  is related to constructing $\CC$ and $\WW$ in step 1 of  both standard and modified Nystr\"om methods. Moreover, given these two matrices, the cost of finding the best rank-$r$ approximation in our proposed method is only slightly higher than the standard Nystr\"om method, as explained in Section \ref{sec:improved-nys}. Thus, the overall computational cost of modified method is almost identical to the standard Nystr\"om method and our method achieves superior performance as demonstrated in  Figures \ref{fig:err_trace_4} and \ref{fig:err_frob_4}.

\begin{figure}[tbhp]
	\centering
	\includegraphics[width=0.7\textwidth]{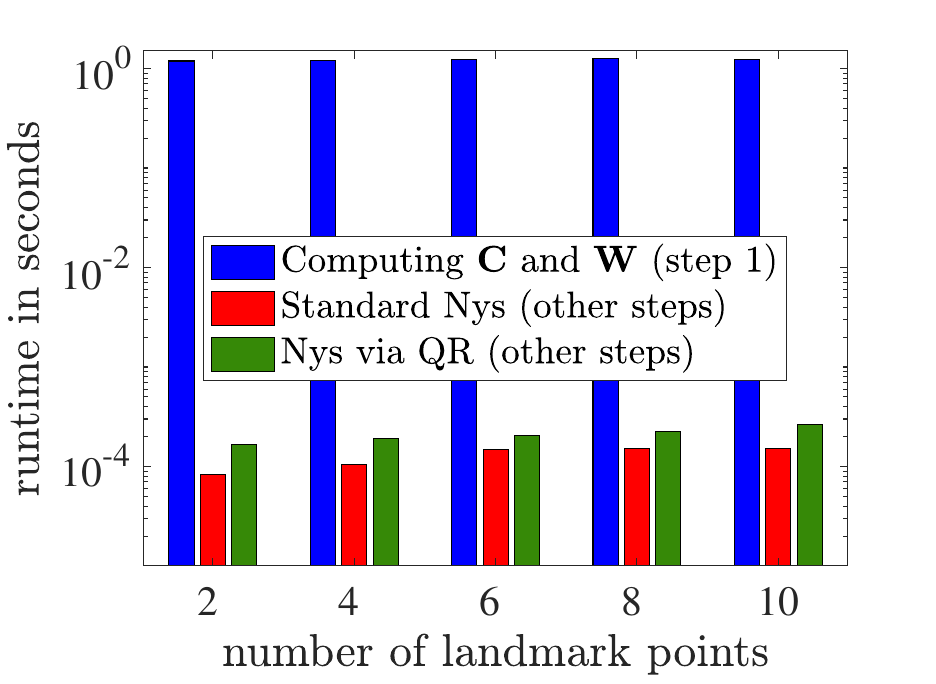}
	\caption{Running time results for standard Nystr\"om and the modified technique.}
	\label{fig:runtime}
\end{figure}

\subsection{Application to Kernel Clustering}
In the last experiment, we demonstrate the performance of modified Nystr\"om on a kernel K-means clustering task. We use the \dataset{segment} data set ($p=19$ and $n=2,\!310$)  from the LIBSVM archive that consists of $7$ clusters. We choose the homogeneous polynomial kernel function of order $2$, i.e., $\kappa(\x_i,\x_j)=\langle\x_i,\x_j\rangle^2$. Previously, we considered Gaussian kernel functions and we would like to show the performance of our method for polynomial kernels as well \cite{RFM_Gittens}.
In addition to the two landmark selection techniques used in Section \ref{sec:exp-1}, we consider another recent technique based  on  Determinantal  Point  Processes (DPP) that was introduced in \cite{li2016fast}. Figure \ref{fig:clustering-accuracy} reports the mean and standard deviation of normalized mutual information (NMI) \cite{oh2017deep} over $200$ trials for fixed rank $r=2$ and various number of landmark points, where the standard K-means algorithm is performed on the columns of $\LL^T\in\R^{r\times n}$ after computing the fixed-rank approximation $\K\approx\LL\LL^T$, cf.~\eqref{eq:low-rank-kernel}. NMI is a popular clustering quality metric which ranges from 0 to 1, and larger values of NMI indicate the higher quality of clustering.

As we see in Figure \ref{fig:clustering-accuracy}, the modified  method outperforms standard Nystr\"om when $m>r$, which is necessary for obtaining accuracies that are close to the baseline EVD. Furthermore, it is observed that modified Nystr\"om leads to higher accuracy results as the number of distinct landmark points increases. In particular, when landmark points are K-means centroids in Figure \ref{fig:acc_2}, the mean NMI of our method reaches EVD with very small standard deviation.
However, for all three landmark selection techniques, the standard Nystr\"om method does not necessarily provide better results when more landmark points are used.
\begin{figure}[tbhp]
	\centering
	\subfloat[uniform]{
		\includegraphics[width=0.5\textwidth]{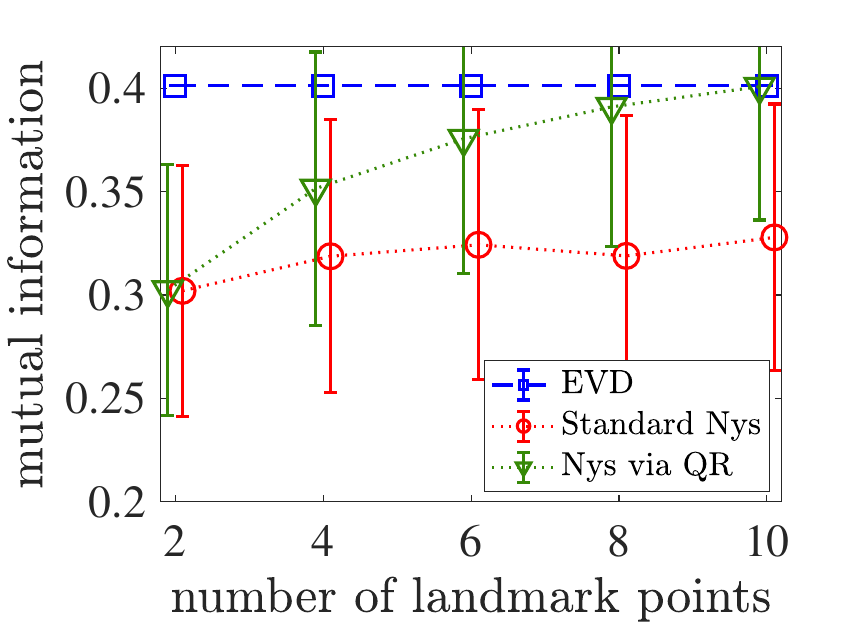}
		\label{fig:acc_1}
	}
	\subfloat[K-means]{
		\includegraphics[width=0.5\textwidth]{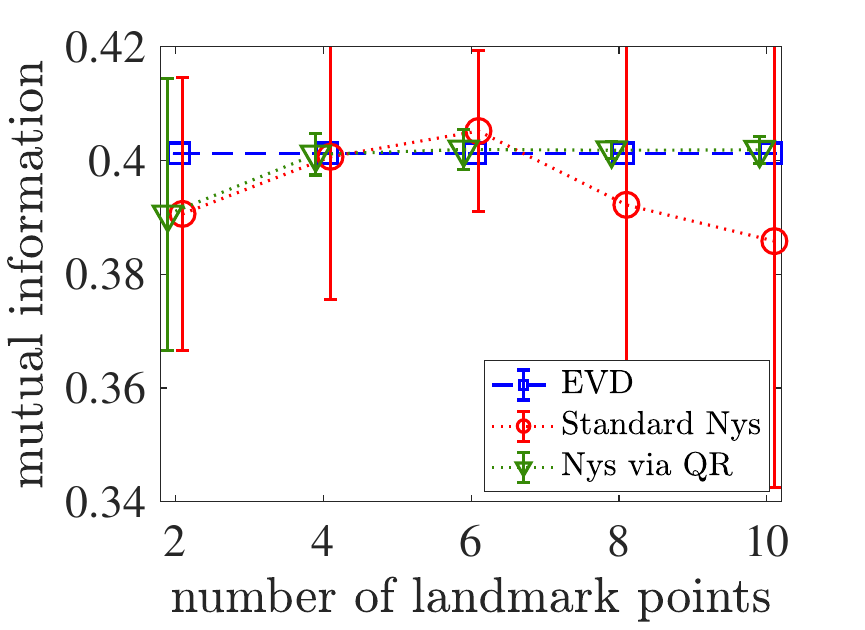}\label{fig:acc_2}
	}
	
	\subfloat[DPP]{
		\includegraphics[width=0.5\textwidth]{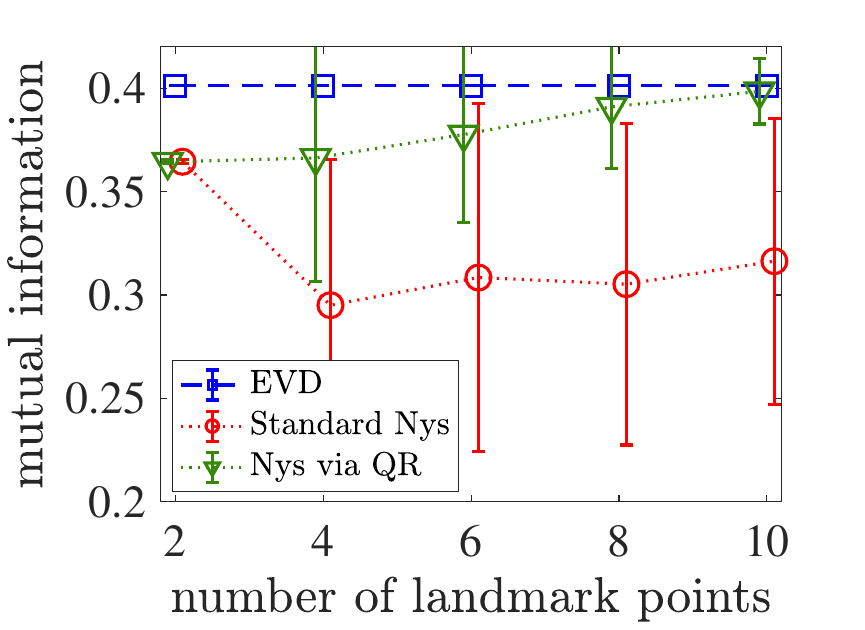}\label{fig:acc_3}
	}
	\caption{Normalized mutual information using three landmark selection techniques.
	}
	\label{fig:clustering-accuracy}
\end{figure}

To further illustrate the importance of modified method, we plot the columns of $\LL^T$ (known as virtual samples \cite{golts2016linearized}) that correspond to two of the $7$ clusters in Figure \ref{fig:clustering-visual}. In this case, $m=10$ landmark points are sampled uniformly at random from the input data. We see that the EVD and our modified method have almost identical results. However, the standard Nystr\"om method spreads out the virtual samples, which justifies its lower accuracy compared to the modified Nystr\"om method and the best rank-$2$ approximation obtained via EVD. 

\begin{figure}[tbhp]
	\centering
	\subfloat[EVD]{
		\includegraphics[width=0.5\textwidth]{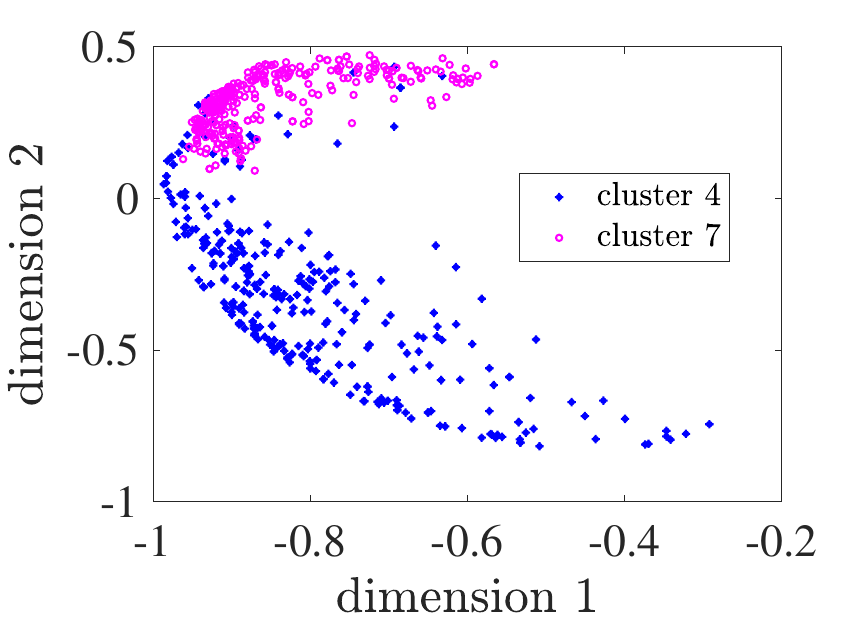}
	}
	\subfloat[Nystr\"om via QR]{
		\includegraphics[width=0.5\textwidth]{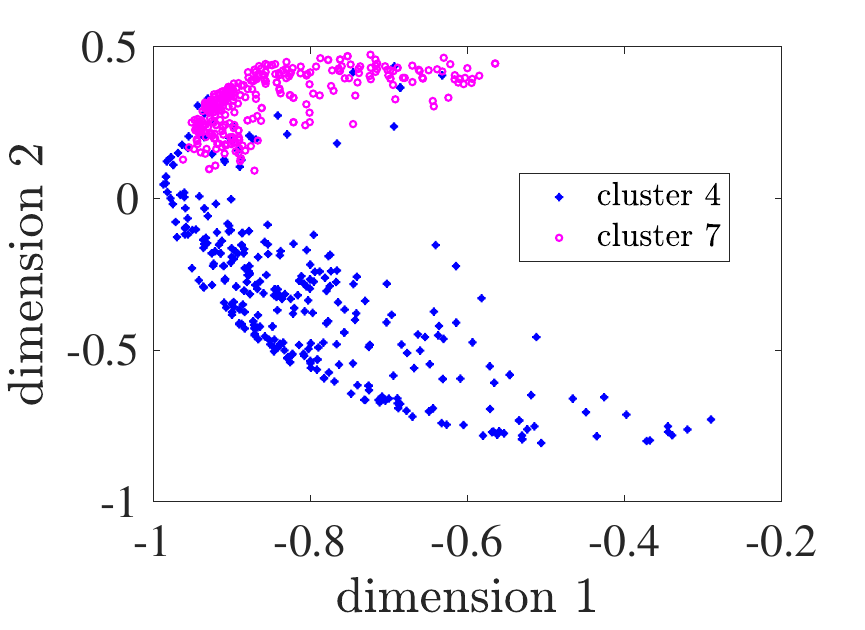}
	}
	
	\subfloat[Standard Nystr\"om]{
		\includegraphics[width=0.5\textwidth]{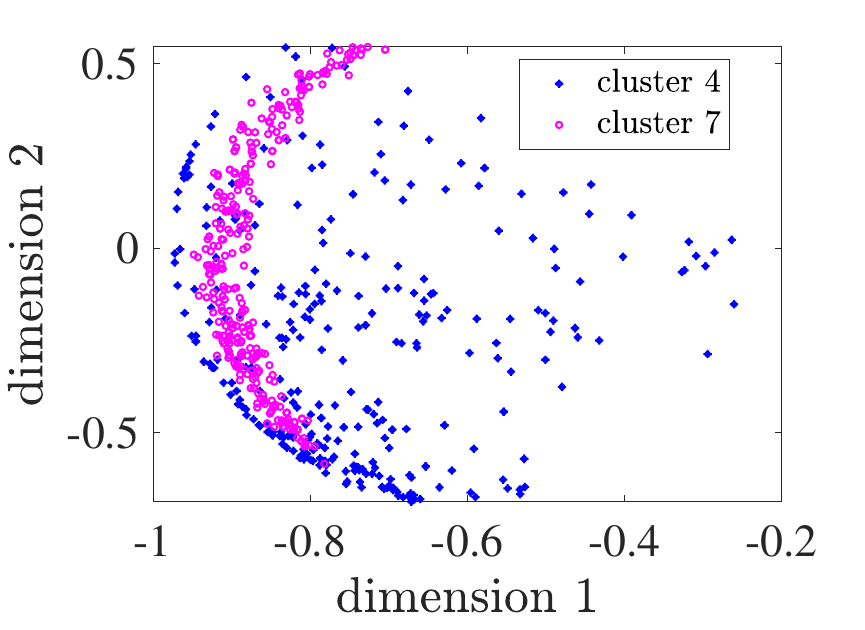}
	}

	\caption{Visualization of the columns of $\LL^T$ that correspond to clusters $4$ and $7$.
	}
	\label{fig:clustering-visual}
\end{figure}

\section{Conclusion}\label{sec:conclusion}
In this paper, we have presented a modified technique for the important process of rank reduction in the Nystr\"om method. Theoretical analysis shows that: (1) the modified method provides improved fixed-rank approximations compared to standard Nystr\"om with respect to the trace norm; and (2) the quality of fixed-rank approximations generated via the modified method improves as the number of distinct landmark points increases. Our theoretical results are accompanied by illustrative numerical experiments comparing the modified method with standard Nystr\"om. We also showed that the modified method has almost the same computational complexity as standard Nystr\"om, which makes it suitable for large-scale kernel machines.


\bibliographystyle{plain}
\bibliography{phd_farhad}

\end{document}